\documentclass{article}
\usepackage{mathbbol}
\usepackage{microtype}
\usepackage{mathtools}
\usepackage{graphicx}
\usepackage{booktabs} 
\usepackage{enumerate}
\usepackage{graphicx} 
\usepackage{float} 
\usepackage{subcaption} 
\usepackage{hyperref}

\usepackage{subcaption}

\usepackage[accepted]{icml2024}

\usepackage{amsmath}
\usepackage{amssymb}
\usepackage{mathtools}
\usepackage{amsthm}

\usepackage[capitalize,noabbrev]{cleveref}

\theoremstyle{plain}
\newtheorem{theorem}{Theorem}[section]
\newtheorem{proposition}[theorem]{Proposition}
\newtheorem{lemma}[theorem]{Lemma}
\newtheorem{corollary}[theorem]{Corollary}
\theoremstyle{definition}
\newtheorem{definition}[theorem]{Definition}
\newtheorem{assumption}{Assumption}
\theoremstyle{remark}
\newtheorem{remark}[theorem]{Remark}

\usepackage[textsize=tiny]{todonotes}
\usepackage{enumitem}

\icmltitlerunning{Learning from Streaming Data when Users Choose}

\begin{document}

\twocolumn[
\icmltitle{Learning from Streaming Data when Users Choose}



\icmlsetsymbol{equal}{*}

\begin{icmlauthorlist}
\icmlauthor{Jinyan Su}{cornell}
\icmlauthor{Sarah Dean}{cornell}
\end{icmlauthorlist}

\icmlaffiliation{cornell}{Department of Computer Science, Cornell University}

\icmlcorrespondingauthor{Jinyan Su}{js3673@cornell.edu}
\icmlcorrespondingauthor{Sarah Dean}{sdean@cornell.edu}

\icmlkeywords{Machine Learning, ICML}

\vskip 0.3in
]



\printAffiliationsAndNotice{}  

\begin{abstract}
In digital markets comprised of many competing services, each user chooses between multiple service providers according to their preferences, and the chosen service makes use of the user data to incrementally improve its model. The service providers' models influence which service the user will choose at the next time step, and the user's choice, in return, influences the model update, leading to a feedback loop. In this paper, we formalize the above dynamics and develop a simple and efficient decentralized algorithm to locally minimize the overall user loss. Theoretically, we show that our algorithm
asymptotically converges to stationary points of of the overall loss almost surely. We also experimentally demonstrate the utility of our algorithm with real world data.

\end{abstract}

\section{Introduction}
 Online services, ranging from social media and music streaming to chatbots and search engines, collect user data in real time to make small adjustments to the models they use to serve and personalize content. 
 Such digital platforms must contend with the demand for instant action, handle continuous data streams, and update their models in an incremental manner. For example, music streaming services continuously process new feedback from user interaction data to refine and enhance their personalized playlist and recommendation models \cite{prey2016musica, eriksson2019spotify, anderson2013popular, morris2015selling, webster2023promise}. Search engines analyze user queries and click through rates to personalize future search suggestions \cite{yoganarasimhan2020search, bi2021leveraging, ustinovskiy2013personalization}. Chatbots learn from user interactions to provide more accurate and context-aware responses over time \cite{clarizia2019context, shumanov2021making, ma2021one}.

Moreover, due to the data-driven nature of digital platforms, interesting dynamics emerge among users and service providers: on the one hand, users choose amongst  providers based on the quality of their services; on the other hand, providers use the user data to improve and update their services, affecting future user choices~\cite{ginart2021competing,kwon2022competition, dean2022multi, jagadeesan2023competition}. For example, in personalized music streaming platform, a user chooses amongst different music streaming platforms based on how well they meet the user's needs. Data from the user's interaction with the platform (such as the music user searches for, saves, or skips) can be used to update its recommendation model in order to better predict users' listening habits and create personalized playlists. The newly updated model affects how well the platform will meet a new user's needs, impacting the future user choice.

In this paper, we study the dynamics of such interactions between users who  choose and services which update their models.
Our focus is on streaming data, meaning that data points arrive sequentially, uncoordinated services, meaning that data is not shared, and imperfectly rational users, who do not always select the best performing service.
In particular, we study a range of imperfect user behaviors to account for the fact that while users prefer better performing services, they might make mistakes or have limited information. 
The degree of imperfection is  characterized by the parameter $\zeta$: with probability $\zeta$, users choose amongst service providers uniformly at random, while with probability $1-\zeta$, they choose the best performing model (i.e. the one with the lowest loss).
This setting is challenging due to the fact that services have only limited information: they observe only the data points of users who choose them.
Indeed, each service must contend with sampled data from a high non-stationary distribution.
Despite these challenges, we propose a simple decentralized algorithm, Multi-learner Streaming Gradient Descent (MSGD), and show that it converges to fixed points with desirable properties.

Our analysis rests on the observation that user sub-populations are naturally induced by the selection between models.
Of course, these sub-populations are highly non-stationary and evolve in feedback with model updates.
Prior work shows that the coupled evolution of users and models gives rise to nonlinear dynamics with multiple equilibria.
\citet{ginart2021competing} and \citet{dean2022multi}
empirically and theoretically demonstrate that services will \emph{specialize} when they repeatedly retrain their models on distributions induced by user selection dynamics.
However, this prior work does not adequately handle the realistic setting in which user data is sampled from the population and arrives in a streaming manner.
When services only observe data from a single user at each time step, they have only partial information about the loss, so model updates cannot guarantee monotonic improvements in performance.
Our key insight is to connect streaming data and user choice to induced sub-populations.
This allows us to analyze our algorithm with tools originally developed in the context of stochastic gradient descent.

In summary, our contributions are:
    (1) We formalize the dynamics of streaming users choosing amongst multiple service providers using the notion of induced sub-populations. By considering imperfectly rational users and streaming data, our setting is general and practical. 
    (2) We provide Multi-learner Streaming Gradient Descent (MSGD), an intuitive and efficient algorithm in which service providers simply perform a step of gradient descent with the \textbf{single user loss} when they are chosen. 
    (3) We theoretically prove that the proposed algorithm converges to the local optima of the \textbf{overall loss function}, which quantifies the social welfare of users. We empirically support our result with experiments\footnote{Code can be found at \url{https://github.com/sdean-group/MSGD}} on real data.
\section{Related Work}
We discuss three strands of related work.

First, learning from non-stationary distributions has roots in
\textit{concept drift}, which studies the problem of learning when the target distribution drifts over time \cite{bartlett1992learning, bartlett2000learning, kuh1990learning, gama2014survey}.
For arbitrary sources of shifts, it is difficult to creating unified objective \cite{gama2014survey, webb2016characterizing}. \textit{Performative prediction} \cite{perdomo2020performative} simplifies the problem by assuming the distribution is induced by the deployed model. 
Most closely related to our setting is multi-player performative prediction \cite{li2022multi, narang2023multiplayer, piliouras2023multi}. 
However, the shifts considered in this literature do not adequately model the partition distributions induced by (bounded) rational user choice.

Second, \textit{learning when users choose} has been studied from several perspectives, including opting-out~\cite{hashimoto2018fairness,zhang2019group}, 
data consent~\cite{kwon2022competition,james2023participatory}, competition between strategic learners~\cite{ben2017best,ben2019regression,jagadeesan2023competition,jagadeesan2023improved}, and strategic users~\cite{shekhtman2024strategic}.
Most closely related our work are papers by \citet{ginart2021competing,dean2022multi,bose2023initializing}
who characterize the specialization that results from the combination of user choice and model retraining. 
These works form a conceptual foundation for our paper, but they do not provide insight into the streaming data setting that we study. 

Third and lastly, \textit{learning from streaming samples} has been studied extensively, with many algorithms based on gradient descent~\cite{yang2021simple, wood2021online}.
Of particular relevance to our analysis is work  on stochastic gradient descent for nonconvex functions~\cite{arous2021online, li2019convergence, cutkosky2019momentum}, distributed stochastic gradient descent \cite{swenson2022distributed, cutkosky2018distributed}
and in particular works connecting this perspective to clustering algorithms~\cite{so2022convergence, tang2017convergence, cohen2021online, liberty2016algorithm, tang2016lloyd}, which share similar structure to the specialization that results from MSGD.

\section{Problem setting}
In this section, we formalize the interaction dynamics between users and service providers.

\textbf{Notation}\quad
Let $\mathcal{P}$ be the data distribution on user data space $\mathcal{X}\subseteq \mathbb{R}^d$.
Let $x$ be a data point of $d$ dimensions, i.e., $x\in \mathbb{R}^d$. Denote $x\sim \mathcal{P}$ if data is drawn from distribution $\mathcal{P}$.
Without loss of generality, we assume that $\mathcal{P}$ is a density.  Given a set $S\subset \mathcal{X}$ with positive probability mass $\mathcal{P}(S)>0$, we denote the distribution obtained by restricting $\mathcal{P}$ onto $S$ as $\mathcal{P}|_{S}$. 

Denote the tuple of $k$ services providers' models by $\Theta=(\theta_1, \cdots, \theta_k)$. For the sake of presentation, we slightly abuse 
notation and refer the model parameter $\theta_i$ as service provider $i$'s model. Given a loss function $\ell$, the loss of model $\theta$ for a user with data $x$ is $\ell(x, \theta)$.
The loss measures the performance of model $\theta$ for user $x$, with low loss corresponding to good performance. 
For example, user data $x=(z,y)$ may contain both features $x$ and a label $y$, and $\theta$ parameterized a model which predicts the label from the features, e.g. $\theta^\top z$. For a regression problem with squared loss, we would have $\ell((z,y),\theta)=(\theta^\top z-y)^2$. In a classification setting, we could similarly define $\ell$ as logistic loss.

\subsection{User-Service Interaction Dynamics} \label{sec: setting}
\textbf{Streaming Data from Users}\quad
Data comes from users sequentially: at each time step $t$, a user $x^t\sim\mathcal{P}$ selects among service providers according to their preferences, and commits their data to the selected provider. 
Our notion of data and user is rather general: at one end of the spectrum, all data could come from a single individual, who distributes specific tasks among different service providers.
At the other end, each data point could come from a different individual within a larger population. 

User preferences for different models can be evaluated by the loss functions $(\ell(x^t, \theta^t_1), \cdots, \ell(x^t, \theta^t_k))$. If users were perfectly rational and had perfect information, they would always choose the model with the lowest loss function. However, considering humans may not have full information when making decisions and the fact that humans sometimes make mistakes, we study a more generalized setting where users may have only bounded rationality. Instead of always choosing the model with the lowest loss, we allow users to make mistakes: with probability $\zeta$, they choose randomly among all the models, while with probability $1-\zeta$, they choose the model that suits them best. This randomness captures the fact that, unlike algorithms and computers, humans are not always stable when making decisions. They may choose randomly because they don't know how to make a choice (due to limited information) or because they don't bother to do the optimization (they just don't care much about which service provider to choose).
We refer to this mode of user behavior as a \textbf{no preference user} to describe user who \textit{``has trouble thinking straight or taking care for the future but who at the same time is actuated by a concern with being fair to other people"} \cite{posner1997rational}.
When users select models according to their preference, we call them a \textbf{perfect rational user}, who \textit{chooses the best means to the chooser's ends} \cite{posner1997rational}.

We remark that there are other user behavior models in the literature such as the Boltzmann-rational model \cite{ziebart2010modeling, luce2005individual, luce1977choice}  where users choose proportionally to $e^{-\alpha\ell(x, \theta_i)}$, and $\alpha$ controls the user rationality. Though not in scope of our present results, we provide further discussion on this behavior model in Section~\ref{subsec:grad}, as this could be of interest to consider in future work.

\textbf{Model Updates by Services}\quad
At each time step $t$, once the user makes a choice, only the chosen service provider $i$ receives the data $x^t$. In general, services have no information about the user population $\mathcal{P}$ other than the data points of users who selected them.
This differs from the usual streaming setting where models see every sampled data point and user choice plays no role.
It is also distinct from the setting in which models receive more than a single sample and can estimate the full distribution of users choosing them. Unlike the usual streaming setting, services cannot repeatedly sample from the same distribution because users choices change in feedback. 
Indeed, services only observe a single sample from time-varying distributions, which is not enough to estimate the distribution. 
As a result of these challenges, it is natural to consider a streaming algorithm that immediately and incrementally updates models based on each observed data point. We introduce such an algorithm in Section 3.

Once the selected service updates its model based on data it receives, the same user-service dynamics repeats at the next time step $t+1$.
The new data point $x^{t+1}$ arrives and is assigned to a service provider based on user choice over the new models $\Theta^{t+1}=(\theta_1^{t+1},\dots,\theta^{t+1}_k)$.

\subsection{Learning Objective}

Given the interaction between model quality and user choice, what is the right learning objective?
In the traditional setting, machine learning aims to minimize the expected loss over the entire population: $\mathbb E_{x\sim \mathcal P}[\ell(x,\theta)]$.
However, this fails to account for the fact that users tend to choose the best performing model.
Motivated by the goal of providing users with the highest quality services we aim to minimize the average loss experienced by users, which we refer to as the \emph{overall loss function}.
This objective can be understood as minimizing the loss of the distribution induced by the parameters, similar to the notion of performative optimality introduced by 
~\citet{perdomo2020performative}.

For ease of presenting the overall loss function, we now introduce the following notation related to the subpopulations induced by $\Theta$:
$X(\Theta) = (X_1(\Theta), X_2(\Theta),\cdots, X_k(\Theta))$
is the data partitioning on $\mathcal{X}$ induced by $\Theta$, where $X_i(\Theta)$ is the set $\{x: i \in \arg\min_{j\in [k]} \ell(x, \theta_j)\mid\Theta\}$. Let 
$a(\Theta) = (a_1(\Theta), \cdots, a_k(\Theta))$ be the proportion of the population $\mathcal P$ contained in $X_i(\Theta)$. Naturally, we have $\sum_{i=1}^k a_i(\Theta) =1$.
Finally, $\mathcal{D}(\Theta) = (\mathcal{D}_1(\Theta), \cdots, \mathcal{D}_k(\Theta))$ is the distribution within each partition, where each $\mathcal{D}_i(\Theta)$ is obtained by restricting $\mathcal{P}$ onto partition $X_i(\Theta)$, i.e., $\mathcal{D}_i(\Theta) = \mathcal{P}|_{X_i(\Theta)}$.
Note for fixed $\Theta$, $X(\Theta)$, $a(\Theta)$ and $\mathcal{D}(\Theta)$ are all fixed. 

To make the models identifiable, and to ensure the above partitions/distribution notations being well-defined, we additional assume the service providers don't have exact the same model. 
\begin{assumption}\label{ass: not the same model}
$
\forall i, j\in [k], \theta_i\neq \theta_j, \forall i\neq j.$ 
\end{assumption}

This assumption says that the services providers have different parameters, which is realistic considering that the service providers don't share parameters. 

Using this notation of induced distributions, we can formally define an objective function that prioritizes user experience
We will start  by defining the learning objective under rational user behavior and then extend it to bounded rationality setting. 

\textbf{Perfect Rationality}\quad
For perfectly rational users and  a fixed $\Theta$, model $i$ is faced with data sampled $x\sim \mathcal{D}_i(\Theta)$. Namely, $x$ is sampled from a distribution supported on $X_i(\Theta)$, where all the $x\in X_i(\Theta)$ prefer model $\theta_i$ over all the other models $\theta_j$ for $j\in [k]$ with $j\neq i$. 
The set $X_i(\Theta)$ can be understood as a subpopulation naturally induced  by the models $\Theta$ and user preference; that is, users choosing the same service provider make up a single subpopulation. 
The expected loss of service provider $i$ over the induced subpopulation can be written as $\mathbb{E}_{x\sim \mathcal{D}_i(\Theta)}[\ell(x, \theta_i)]$. 
In order to write the overall loss function, we need to consider all models.
Since $a_i(\Theta)$ is the portion of subpopulation $X_{i}(\Theta)$ within the total population $\mathcal{P}$, the overall learning objective under perfect rationality can be written as
\begin{equation}\label{eq: 1}
\begin{aligned}
f_{\text{PR}}(\Theta)=&\sum_{i=1}^ka_i(\Theta)\cdot \underset{x\sim \mathcal{D}_i(\Theta)}{\mathbb{E}}[\ell(x,\theta_i)].
\end{aligned}
\end{equation}

This expression can also be understood as the expected loss over users $x$ sampled from the population $\mathcal P$ and models $\theta$ sampled according to rational user choice.
From this perspective, the expression in~\eqref{eq: 1} is the result of applying the tower property of expectation and conditioning on the model choice being $i$.
Now we extend this loss function to bounded rationality. 

\textbf{Bounded Rationality}\quad
A perfectly rational user would deterministically choose service $\theta_i$ to minimize their loss $\ell(x, \theta)$. However, due to the complexity and uncertainty in real world decision making, such as limited knowledge, resource, and time, users don't always act to maximize their utility \cite{selten1990bounded, jones1999bounded, gigerenzer2002bounded}. As a result, they might just randomly choose a service provider due to limited time or imperfect information. We refer to this as a user with \textit{bounded rationality}. We introduce the parameter $\zeta$ to characterize this user behavior:
with probability $1-\zeta$, users rationally choose the model $\theta_i$ which minimizes their loss,
and with probability $\zeta$, they choose uniformly at random amongst all the models $\Theta=(\theta_1, \cdots, \theta_k)$.
Conditioned on this latter event, the probability that the user selects model $i$ is $1/k$ for all $i\in[k]$.
Conditioned on a perfectly rational choice, the probability that the user selects model $i$ is zero unless $x\in X_i(\Theta)$. 
Combining these two events, the expected loss for users sampled from population $\mathcal P$ and models sampled according to the user choice is
\begin{equation}\label{eq: 2}
\begin{aligned}
f(\Theta) 
=&  \underset{x\sim \mathcal P}{\mathbb{E}}\Big[\sum_{i=1}^k \Big ((1-\zeta) \mathbb 1_i(x,\Theta) +\frac{\zeta}{k}\Big)\ell(x,\theta_i)\Big]\:,\\
\end{aligned}
\end{equation}
where we define $\mathbb 1_i(x,\Theta)=\mathbb 1\{x\in X_i(\Theta)\}$.
This expected loss defines the learning objective in the bounded rationality setting.
Compared to perfect rationality, the learning objective under bounded rationality accounts for the fact that users may not always select the best service for them.

\section{Multi-learner Streaming Gradient Descent}
In this section, we investigate important properties of the learning objective and then present an algorithm for the multi-learner setting which works for both perfect and bounded rational users.

\subsection{Properties of Learning Objective}\label{sec: boundedness}

We begin by deriving some important properties of the learning objective $f(\Theta)$, laying the groundwork for the subsequent sections. Formal proofs of these properties are provided in Appendix \ref{app: properties}.

\textbf{Property 1: Decomposability.} \quad
With some simple algebraic manipulation of $f(\Theta)$, we can decompose $f(\Theta)$ as
$f(\Theta) = (1-\zeta) \cdot f_{\text{PR}}(\Theta) +\zeta \cdot f_{\text{NP}}(\Theta)$, where $$\textstyle f_{\text{NP}}(\Theta) = \frac{1}{k}\sum_{i=1}^k\underset{x\sim \mathcal{P}}{\mathbb{E}}[\ell(x,\theta_i)]$$
represents the loss function when users have no preference over the services providers. 
When $\zeta = 0$, the learning objective reduces to perfect rationality $f_{\text{PR}}(\Theta)$; when $\zeta = 1$, the it reduces to ``no preference" $f_{\text{NP}}(\Theta)$.

\textbf{Property 2: Boundedness.}\quad
$f(\Theta)$ is lower bounded by $f_{\text{PR}}(\Theta)$ and upper bounded by $f_{\text{NP}}(\Theta)$, i.e., $f_{\text{PR}}(\Theta)\leq f(\Theta)\leq f_{\text{NP}}(\Theta)$. 

We now make a set of common assumptions on the loss as a function of parameters: $\ell(x, \cdot)$.
\begin{assumption}\label{ass: loss function}
Assume for all $x\in\mathcal X$, the loss function $\ell(x, \cdot)$ is non-negative, convex, differentiable, $L$-Lipschitz, and $\beta$-smooth. 
\end{assumption}
\begin{figure}[b]

\centerline{\includegraphics[width=0.7\columnwidth]{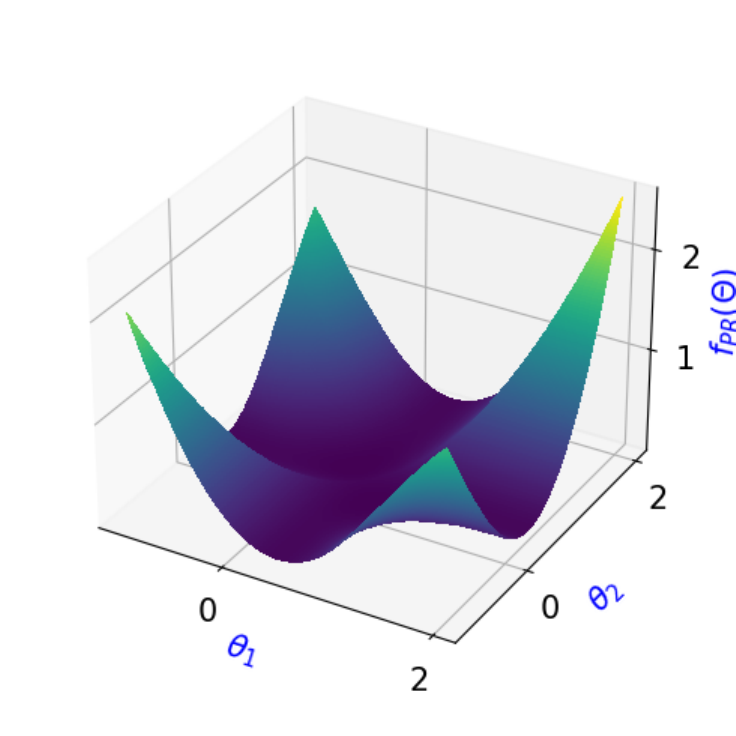}}
\caption{Example of $f_{\text{PR}}(\Theta)$ being non-convex.}
\label{fig: showing f(PR) is non-convex}
\end{figure}
\textbf{Property 3: Non-convexity.}\quad
When $\ell(x, \theta)$ is convex in $\theta$, $f_{\text{NP}}(\Theta)$ is convex in $\Theta$, but since $f_{\text{PR}}(\Theta)$ is generally non-convex,  $f(\Theta)$ is also non-convex when $\zeta<1$.

In Figure \ref{fig: showing f(PR) is non-convex}, we give a simple example in which the loss $f_{\text{PR}}(\Theta)$ is non-convex. In this example, we let $\mathcal P$ be a uniform distribution over the interval $[0, 1]$, the number of services $k=2$, and the loss function to be $\ell(x, \theta) = (x-\theta)^2$. Full details are provided in Appendix \ref{app: illustration of f being non-convex}. 

Since $f(\Theta)$ is generally non-convex, we aim to design algorithms whose outputs converge to \emph{local optima}, i.e. the stationary points of learning objective $f(\Theta)$.
\begin{definition}(Stationary Points)\label{def: stationary point}
The stationary points of a differentiable function $f(\Theta)$ are $\{\Theta: \nabla f(\Theta) = 0\}$.
\end{definition}
\begin{remark}
{While stationary points are local minima of the $f(\Theta)$ objective, they are \textbf{global} minima for a decoupled objective where the effect of $\Theta$ on the partition is not accounted for. As we will show in the proof of Lemma \ref{lemma: gradient of learning objective}, stationary points contains the points where each model is at the global optimum of the induced distribution of users that they see.}
\end{remark}

\subsection{Multi-Learner Streaming Gradient Descent}

We propose an intuitive and simple algorithm, Multi-Learner Streaming Gradient Descent (MSGD), presented in Algorithm \ref{alg: simplified single model update}. 
\begin{algorithm}[tb]
	\caption{Multi-learner Streaming Gradient Descent (MSGD) \label{alg: simplified single model update}}
	\begin{algorithmic}
		\STATE {\bfseries Input:} Rationality parameter $\zeta$; loss function $\ell(\cdot,\cdot)\geq 0$; Initial models $\Theta^0 = (\theta^0_1, \cdots, \theta^0_k)$; Learning rate $\{\eta^{t}\}_{t=1}^{T+1}$.
       \FOR{$t=0, 1,2, \dots, T$ }
\STATE   Sample data point $x\sim \mathcal{P}$,

\STATE\textcolor{lightgray}{User Side: \textbf{User Selects a Service Provider:}}\\
\STATE
User selects best model $i=\arg\min_{j\in [k]} \ell(x, \theta_j^t)$ w.p. $1-\zeta$, otherwise the user selects $i$ from $\{1,...,k\}$ uniformly at random w.p. $\zeta$. 
\STATE\textcolor{lightgray}{Learner Side: \textbf{Selected Learner Updates its Model:}}\\
\STATE   The selected model $i$ receives data and performs a gradient step: 
$\theta^{t+1}_{i}= \theta^t_{i}-\eta^{t+1}\cdot\nabla\ell(x,\theta^t_i)$, while all other models remain the same $\theta^{t+1}_j=\theta^t_j$ for all $j\neq i$.

        \ENDFOR\\
        \STATE \textbf{Return }
        $\Theta^T$
	\end{algorithmic}
\end{algorithm}
We write the algorithm to reflect the setting (Section \ref{sec: setting}) from both the user and the learner side. 

On user side, a user comes into the digital platform at each time step. The user selects a model amongst service providers according to their preference and rationality.
The user shares their data only with the selected service.
We emphasize that these steps reflect the user behavior, which is not under the control of service providers or algorithm designers.
On the learner side, one service will receive data; this service computes the gradient of the loss for this single user and then performs a gradient descent step. Because the services are not coordinated, the parameters of the models that were not selected remain the same. We denote model parameter tuple at time $t$ as $\Theta^t = (\theta_1^t, \cdots, \theta_k^t)$. At each time step, the selected model $i$ updates with a gradient step $\theta_i^{t+1}\leftarrow \theta_i^t-\eta^{t+1}\cdot \nabla \ell(x, \theta_i^t)$. In the next time step $t+1$, a new user will arrive and decide which model to select based on the updated model parameters $\Theta^{t+1}$. 

MSGD is practical due to three main advantages. First, it is computationally affordable, memory efficient and privacy-conscious. At each time step, when a data point arrives, the service only has to perform a lightweight gradient descent step, which enables services to adapt quickly to incoming information without using extensive computational resources. Moreover, no extra storage is needed to retain the past user data, which may also address privacy concerns. Second,  MSGD is amenable to the partial information setting: services do not need to know anything other than the data they receive. Third, MSGD handles non-stationary user distribution: user preferences update along with model parameters. Despite the overall user population being constant, the subpopulation that will choose a specific model evolves over time. To enhance user experience, service providers optimize over the population that chooses them, i.e., $\mathcal{D}_i(\Theta)$, which not a static distribution but a function of $\Theta$.

Although the gradient update from learner side in Algorithm \ref{alg: simplified single model update} is intuitively simple, it is not straightforward to see whether Algorithm \ref{alg: simplified single model update} will perform well with respect to the overall objective $f(\Theta)$. 
One challenge arises due to the fact that $f(\Theta)$ is non-convex. Even focusing on convergence to local optima (Definition~\ref{def: stationary point}) leaves several additional challenges.
First, notice that instead of updating all the learners at the same time with batched data, in MSGD, learners are updated asynchronously, one at a time depending on user choice, fostering the potential for competition amongst learners.
Indeed, an update to model $i$ can affect the distribution of users selecting any model $j\neq i$, meaning that each model must content with a highly non-stationary distribution.
Moreover, since the update uses the gradient of a single data point, rather than the gradient of the objective $f(\Theta)$, we can't guarantee that the update will decreases $f(\Theta)$ (which would imply convergence to
a local optimum).
It might be natural to consider something like the \emph{learner expected loss} $\mathbb E_{x\sim \mathcal D_i(\Theta)}[\ell(x,\theta_i)]$ which would measure the \emph{local} performance of each model.
However, due to the streaming nature of the update, it is still not possible to show that this learner loss  will decrease.
In other words, the streaming update step in Algorithm \ref{alg: simplified single model update} is not
\textit{risk reducing}, a property that \citet{dean2022multi} leveraged to prove convergence in the full information setting. 

In Section~\ref{sec: convergence}, we theoretically prove the convergence of Algorithm \ref{alg: simplified single model update} to stationary points of the learning objective.
First, in the following subsection, we connect the gradient of a single user $\nabla \ell(x, \theta)$ with the gradient of the learning objective $\nabla f(\Theta)$.
\subsection{Gradient of learning objective $\nabla f(\Theta)$}\label{subsec:grad}
The following lemma computes the gradient of the objective function $f(\Theta)$.
This lemma facilitates the analysis of MSGD.
\begin{lemma}\label{lemma: gradient of learning objective}
For the learning objective 
 $f(\Theta)$ defined in Eq.~\eqref{eq: 2}, the  gradient with respect to $\theta_i$ is:
 \begin{align*}
 \nabla_{\theta_i} f(\Theta) = &(1-\zeta) \cdot a_i(\Theta)\cdot\underset{x\sim \mathcal{D}_i (\Theta)}{\mathbb{E}}[\nabla_{\theta_i}\ell(x,\theta_i)] \\
 &+  \frac{\zeta }{k}\cdot \underset{x\sim \mathcal{P}}{\mathbb{E}}[\nabla_{\theta_i} \ell(x, \theta_i)]\:.
 \end{align*}
\end{lemma}
\begin{proof}[Proof Sketch]
Recall the decomposition of $f(\Theta) = (1-\zeta)\cdot f_{\text{PR}}(\Theta) +\zeta \cdot f_{\text{NP}}(\Theta)$. The gradient of  $f_{\text{NP}}(\Theta)$ is easy to compute, since by linearity of the gradient and Lipschitzness of the loss we may write $\nabla_{\theta_i} f_{\text{NP}}(\Theta) = \frac{1 }{k}\cdot \underset{x\sim \mathcal{P}}{\mathbb{E}}[\nabla_{\theta_i} \ell(x, \theta_i)]$. 
It is more difficult to compute the gradient of $f_{\text{PR}}(\Theta)$ because the distribution $\mathcal D_i(\Theta)$, and thus the domain of integration, is also a function of $\Theta$. In our proof, we calculate $\nabla_{\theta_i}f_{\text{PR}}(\Theta)$ through its directional derivative $D_{v}f_{\text{PR}}(\Theta) = \lim_{\gamma\rightarrow 0}\frac{1}{\gamma}(f_{\text{PR}}(\Theta+\gamma v)-f_{\text{PR}}(\Theta))$. Similar to \citet{so2022convergence}, we decouple the difference $f_{\text{PR}}(\Theta+\gamma v)-f_{\text{PR}}(\Theta)$ into two parts; one has fixed integral domain that is independent of $\Theta$, and the other term is an integration on a zero measure set, which is 0 when we take the limit. In the end, it turns out that the derivative of $f_{\text{PR}}(\Theta)$ can be computed by treating its domain of integral to be fixed. Thus, we can move the derivative inside the integral and get
\begin{equation} \label{eq18}
\nabla_{\theta_i} f_{\text{PR}}(\Theta) = a_i(\Theta)\cdot\underset{x\sim \mathcal{D}_i(\Theta)}{\mathbb{E}}[\nabla_{\theta_i}\ell(x,\theta_i)]
\end{equation}
The complete proof can be found in Appendix \ref{app: calculate gradient}. 
\end{proof}

This lemma shows that the gradient of the objective with respect to model $i$ depends only on the gradient of the loss with respect to model $i$'s parameters.
The decomposability allows for decentralized algorithms, like the one proposed in Algorithm~\ref{alg: simplified single model update}.
We now turn to formalizing the guarantees of this algorithm in terms of convergence.

\begin{remark}
Though it is interesting to consider the Boltzmann-rational model \cite{ziebart2010modeling, luce2005individual, luce1977choice}, this form of user choice poses additional difficulties for the development of decentralized algorithms. 
Under this behavior model, the gradient $\nabla_{\theta_i} f(\Theta)$ w.r.t. model $i$ unavoidably depends on the loss of other service providers. This poses a challenge to our setting: though service providers are all interested in providing an accurate service, they are not \textbf{coordinated} and can't share information with each other.
This challenge may be of interest for future work. \end{remark}

\section{Asymptotic Convergence Analysis}\label{sec: convergence}
 Define filtration $(\mathcal{F}^t)_{t=0}^{\infty}$ associated with MSGD. In particular, let $\mathcal{F}^0=\sigma(\Theta^0)$ be the $\sigma$-algebra generated by $\Theta^0$.
Let $\sigma$-algebra $\mathcal{F}^t =\sigma (x^t, \eta^t, i^t, \mathcal{F}^{t-1})$  contain all the information up to iteration $t$, where $i^t$ indicates the random user choice. 
We now show that MSGD will converge to stationary points of the learning objective.
Formally, convergence is defined as follows.
\begin{definition}(Convergence)\label{def: convergence}
We say a sequence $\{\Theta^t\}_{t=0}^{\infty}$ converges to a set $\mathcal T$ if $\exists ~T\in \mathbb{N}$, s.t., $\forall t>T, \Theta^t\in \mathcal T$.
\end{definition}
We will prove that MSGD converges under the following additional assumptions on step size, the user population, and the loss function.
These assumptions are standard in the literature on stochastic (non)convex optimization \cite{ge2015escaping, bertsekas2000gradient, wang2019stochastic}.

\begin{assumption}\label{ass: stepsize}
The stepsize $\{\eta^{t+1}\}_{t=0}^{\infty}$ satisfies the  condition: $\sum_{t=0}^{\infty}{(\eta^{t+1})^2}<\infty$. 
\end{assumption}

\begin{assumption}\label{ass: compact}
We assume $\{\nabla f(\Theta)=0\}$ to be compact. 
\end{assumption}
The above compactness assumption is rather common \cite{leluc2020asymptotic, so2022convergence}, which allows us to prove that $\{\Theta^{t}\}_{t=0}^{\infty}$ converges to $\{\nabla f(\Theta)=0\}$ by showing that $\nabla f(\Theta^t)$ converges to 0. 
\begin{assumption}\label{ass: distribution}
    Assume the underlying data distribution $\mathcal{P}$ has a continuous density function $p$ with a bounded support, namely, $\text{Pr}_{x\sim \mathcal{P}}(\|x\|>R) = 0$ for some $R>0$.
\end{assumption}
\begin{assumption}\label{ass: assumptions for proving a being Lipschitz}
For any {$\theta\neq \theta^{'}$, there exists a $d_0$ such that for all small $d<d_0$}, the Lebesgue measure of set $S_d=\{x: |\ell(x, \theta)-\ell(x,\theta^{'})|<d\}$ is bounded by $d$, i.e., $\text{Vol}(S_d)\leq d$. 
\end{assumption}
This assumption states that the loss function  are good enough for most users to distinguish different services (so that they can make a choice) and {a \textbf{sufficiently small} perturbation on one of the models won't dramatically change user preference. } {We provide further intuition and examples in Appendix \ref{appedix: assumption intuition}.}

Under these assumptions, we have the following theorem showing that the learning objective converges.
Moreover,
under an additional condition that $\eta^t$ decreases with a rate of $\frac{1}{t}$, the objective
converges
to a local optimum and the iterates $\{\Theta^t\}_{t=1}^T$ converge to stationary points.  The following theorem makes this formal. 
\begin{theorem}\label{thm: main}(Convergence of Algorithm \ref{alg: simplified single model update})
Denote the iterates from Algorithm \ref{alg: simplified single model update} and their overall loss to be $\{\Theta^t\}_{t=0}^T$ and $\{f(\Theta^t)\}_{t=0}^T$ respectively. 
Under Assumptions \ref{ass: not the same model}, \ref{ass: loss function}, \ref{ass: stepsize} 
there is an $\mathbb{R}$-valued random variable $f^{*}$ such that $f(\Theta^{t})$ converges to $f^{*}$ almost surely. Additionally, under Assumptions \ref{ass: compact}, \ref{ass: distribution}, \ref{ass: assumptions for proving a being Lipschitz} and setting $\eta^t = \frac{\eta_c}{t}$, where $\eta_c$ is a constant, 
the iterate $\{\Theta^{t}\}_{t=1}^T$
 converges to the set of stationary points of $f(\Theta)$, i.e, $\{\Theta:\nabla f(\Theta) = 0 \}$ almost surely.
\end{theorem}
To prove Theorem \ref{thm: main},
we first argue that the objective converges.
Then, we use this fact to show that the parameters converge to a stationary point.
The argument proceeds in the following subsections respectively.

\subsection{Convergence of $f(\Theta)$}

We first provide an outline of the proof of convergence of the learning objective.
To start, we show an analytic upper bound on the value of $f(\Theta)$ at time $t+1$ compared to time $t$.
This bound relies on the smoothness of $\ell(x, \cdot)$.
It writes this difference explicitly in terms of the gradient of the objective function $\nabla f(\Theta)$, as computed in Lemma \ref{lemma: gradient of learning objective}, and  the gradient of a single loss $\ell(x, \theta_i)$, as used for the gradient update in MSGD algorithm. 
Formally, we have the following lemma. 
\begin{lemma}\label{lemma: analytic upper bound of f}
Let $f(\Theta)$ be our learning objective proposed in Eq.~\eqref{eq: 2} and let $\bar\zeta = 1-\zeta$ denote the probability that the best model is selected (while w.p. $\zeta$ a random model is selected). 
Let $i$ be the model selected at time $t$.
Then the following inequality holds under Assumption \ref{ass: loss function}:
\begin{equation*}
f(\Theta^{t+1})\leq f(\Theta^t) -A^{t+1} +B^{t+1}-C^{t+1}
\end{equation*}
where 
\small
\begin{align}
A^{t+1}=&\bar\zeta\eta^{t+1} \frac{\|\nabla_{\theta_i}f_{\text{PR}}(\Theta^t)\|^2}{a_i(\Theta^t)}\notag +\zeta   \eta^{t+1} \|\nabla_{\theta_i}f_{\text{NP}}(\Theta^t)\|^2\notag\\
B^{t+1} =&\frac{\beta}{2} \sum_{i=1}^k (\eta_i^{t+1})^2 \cdot \|\nabla\ell(x^{t+1}, \theta_i^t)\|^2\notag\\
C^{t+1} = &\bar\zeta  \eta^{t+1} \langle \nabla_{\theta_i} f_{\text{PR}}(\Theta^t), \nabla \ell(x^{t+1}, \theta_i^t)-\frac{\nabla_{\theta_i} f_{\text{PR}}(\Theta^t)}{a_i(\Theta^t)}\rangle\notag\\
&+\zeta  \eta^{t+1} \langle \nabla_{\theta_i} f_{\text{NP}}(\Theta^t),\nabla \ell(x^{t+1}, \theta_i^t)- \nabla_{\theta_i} f_{\text{NP}}(\Theta^t)\rangle \notag
\end{align}
\end{lemma}

This lemma quantifies the decrease (or lack thereof) in the objective function value from one time step to the next.
We therefore turn our focus to
the sequences $(A^t)_{t=0}^{\infty}$, $(B^t)_{t=0}^{\infty}$, $(C^t)_{t=0}^{\infty}$.
The sequence $(C^t)_{t=0}^{\infty}$ is $\mathcal{F}^t$ martingale difference sequence because, as we show in Eq.~\eqref{eq18} in the proof of Lemma~\ref{lemma: gradient of learning objective}, the single loss in the MSGD update is parallel to the gradient of the objective $f(\Theta)$ in expectation.
Then also notice that $(B^t)_{t=1}^{\infty}$ converges when  $\sum_{t=0}^{\infty}{(\eta^{t+1})^2}<\infty$ (Assumption~\ref{ass: stepsize}).  

Let $(M^{t})_{t=1}^{\infty} $ be defined by:
$$
M^{t+1} = f(\Theta^{t+1})-\sum_{\tau=0}^t B^{\tau+1}.$$
Then we have that
$M^{t+1}\leq M^t  -A^{t+1} -C^{t+1}$.
We show that $M^t$ is an $\mathcal{F}^t$ super-martingale, which converges almost surely by the martingale convergence theorem. From this, the convergence of $f(\Theta)$ follows by the convergence of $(B^t)_{t = 0}^{\infty}$. The complete proof can be found in Appendix \ref{app: main}
\subsection{Convergence of iterates $\Theta$}
Next, we show that the model parameters $\Theta$ also converge, and in particular that they converge to the stationary points of the overall loss.
The convergence of $\Theta$ follows from the convergence of $f(\Theta)$, under some additional conditions on the step size. To simplify the notation, we denote model update at each round as $\theta_i^{t+1} = \theta_i^t -\eta_i^{t+1}\cdot \nabla\ell(x^{t+1}, \theta_i^t)$, where $\eta_i^{t+1}\neq 0$ only when model $i$ was selected by the user.
We define the following event $\mathcal E_i(\tau, T, r, s)$:
$$\mathcal E_i(\tau, T, r, s) =\left\{\sum_{t=\tau}^{T(\tau)} \eta^{t+1}<r ~~\text{and}~~\sum_{t=\tau}^{ T(\tau)} \eta_i^{t+1}>s\right\}\:.$$
This event occurs when, for a particular time interval, the accumulated step size $\eta^t$ is bounded above by a constant, while the accumulation of steps of model $i$ is bounded below.
With this definition in hand, we have the following lemma.

\begin{lemma}\label{lemma: convergence of iterates}
Under Assumption \ref{ass: not the same model}-\ref{ass: assumptions for proving a being Lipschitz}, 
$\Theta^t$ converges to the stationary point of $f(\Theta)$ almost surely if the following condition on $\eta^t$ holds:
For any $\epsilon $, there is an $r_0>0$ 
such that if $r\in (0, r_0)$, then there exists a mapping $T: \mathbb{N}\rightarrow \mathbb{N}$, a time step $t_0\in \mathbb{N}$, and constants $s,c>0$ 
such that, for any $\tau >t_0$:
\begin{equation*}
\text{Pr}\left(\mathcal E_i(\tau,T,r,s)\bigg|\mathcal{F}^{\tau}, \|\nabla_{\theta_i} f(\Theta^{\tau})\|> \epsilon\right)>c\:.
\end{equation*}
\end{lemma}
The first part of the event $\mathcal E_i(\tau,T,r,s)$, which bounds the accumulated step size $ \eta^{t}$ by $r $, indicates that the gradient $\|\nabla_{\theta_i} f(\Theta^{t})\|> \frac{\epsilon}{2}$, i.e., the gradient remains large in $t\in [\tau, T(\tau))$. (Notice that the total displacement between $\Theta^{\tau}$ and $\Theta^{T(\tau)}$ can be controlled by bounding accumulated learning rate).
The second part $\sum_{\tau\leq t< T(\tau)} \eta_i^{t+1}>s$ ensures that we accumulate enough learning for the center $i\in [k]$ with large
gradients, so that $f(\Theta)$ decreases by at least a constant amount on this interval. 

To finish proving Theorem~\ref{thm: main}, it remains  to set stepsize in Algorithm \ref{alg: simplified single model update} to be $\eta^{t} = \frac{\eta_c}{t}$, and show that this satisfies the condition proposed in Lemma \ref{lemma: convergence of iterates}. 
The complete proof of Theorem \ref{thm: main} is given in Appendix \ref{app: main}.
\section{Experiments}
We experimentally illustrate the performance of MSGD using real world data. 
These experiments confirm our theoretical results and illustrate interesting phenomena that occur in the multi-learner setting.
Code for reproducing these results can be found at \url{https://github.com/sdean-group/MSGD}.

\subsection{Experimental Settings}
We illustrate the performance of MSGD in two distinct settings with different datasets and loss functions.

\begin{figure}[t]
\centerline{\includegraphics[width=0.5\textwidth]{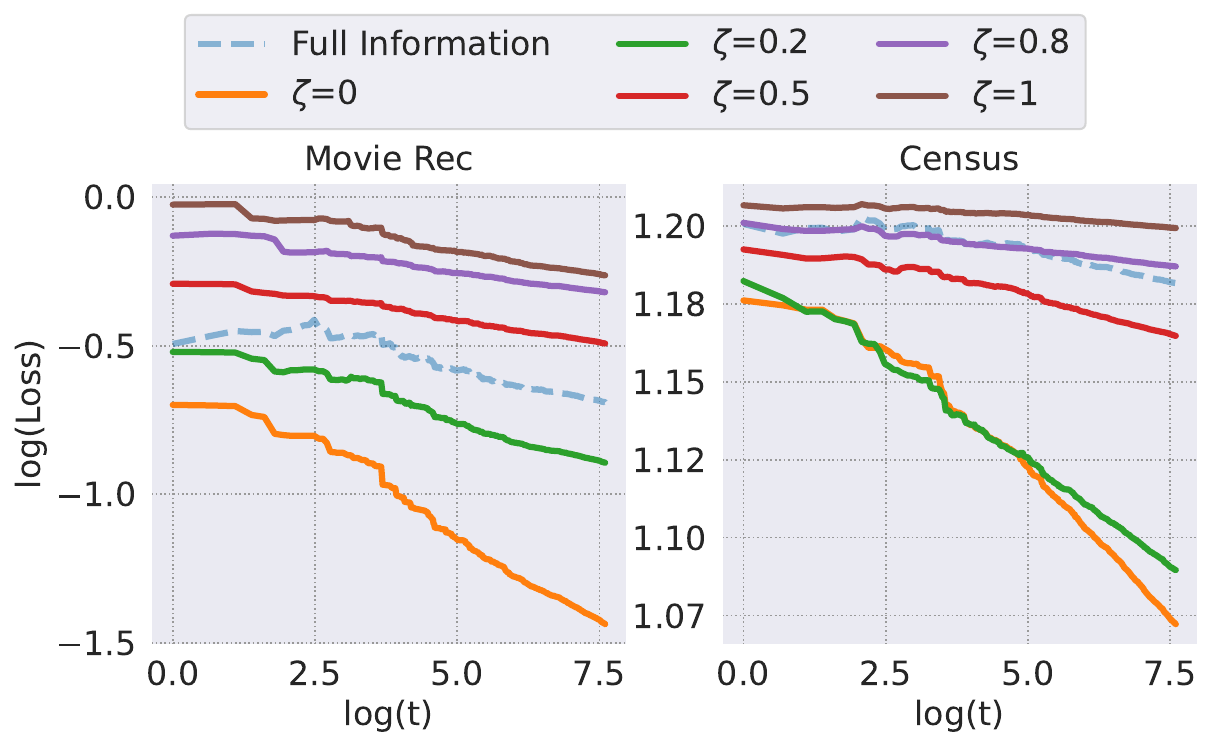}}
\caption{Convergence of objective function $f(\Theta)$ under MSGD or Full Information with $k=3$ services in the movie recommendation (left) and census data (right) tasks. }
\label{fig: convergence of f}
\end{figure}

\textbf{Movie Recommendation with Squared Loss}\quad Our first experimental setting is based on a widely used movie recommendation dataset Movielens10M \cite{harper2015movielens}, which contains 10 million ratings across 10k movies by 70k viewers. We preprocess the dataset with a similar procedure from~\citet{bose2023initializing}: We filter the total ratings and only keep the top 200 most rated movies and use the inbuilt matrix factorization function from Python toolkit \texttt{Surprise} \cite{hug2020surprise} to get $d =5$ dimensional user embeddings. This preprocessing procedure results in a population of 69474 users each with a five dimensional feature vector which we denote by $z$.
Then we consider a regression problem, where each model aims to predict the the user ratings $r$ of the $d_r=200$ movies.
Denote by $x=(z,r)$ each user's data and by $\Omega_x$ the set of movies which have been rated.
Let $m=|\Omega_x|$ denote the number of movies user $x$ has rated.
For each user, we define the following squared loss, where $\theta\in \mathbb{R}^{d\times d_r}$. 
\begin{equation*}
\ell(x, \theta) = \frac{1}{m}\sum_{i\in\Omega_x}(\theta_i^\top z - r_i)^2\:.
\end{equation*}
\textbf{Census Data with Logistic Loss}\quad
Our second setting is based on census data made available by \texttt{folktables}~\cite{ding2021retiring}.
We consider the ACSEmployment task, where to goal is to predict whether individuals are employed. The population is defined by the 2018 census data from Alabama, filtered to individuals between ages 16 and 90, resulting 47777 user in total. After splitting 0.2 of them for testing, we have 38221 data to sample from. 
The data contains $d=16$ features describing demographic characteristics like age, educated, marital status, etc. Denote each user data $x=(z,y)$, where $x\in \mathbb{R}^{d}$ and $y\in \{0,1\}$. We scale features $x$ so that they have zero mean and unit variance. 
We use logistic regression loss for this task, 
\begin{equation*}
\ell(x, \theta) =-y\log(\theta^\top z)-(1-y)\log(1-\theta^\top z)\:.
\end{equation*}
The model predicts $1$ if $\theta^{\top}z>0$ and 0 otherwise. 
\subsection{Results}
\begin{figure}[t]
\centerline{\includegraphics[width=0.5\textwidth]{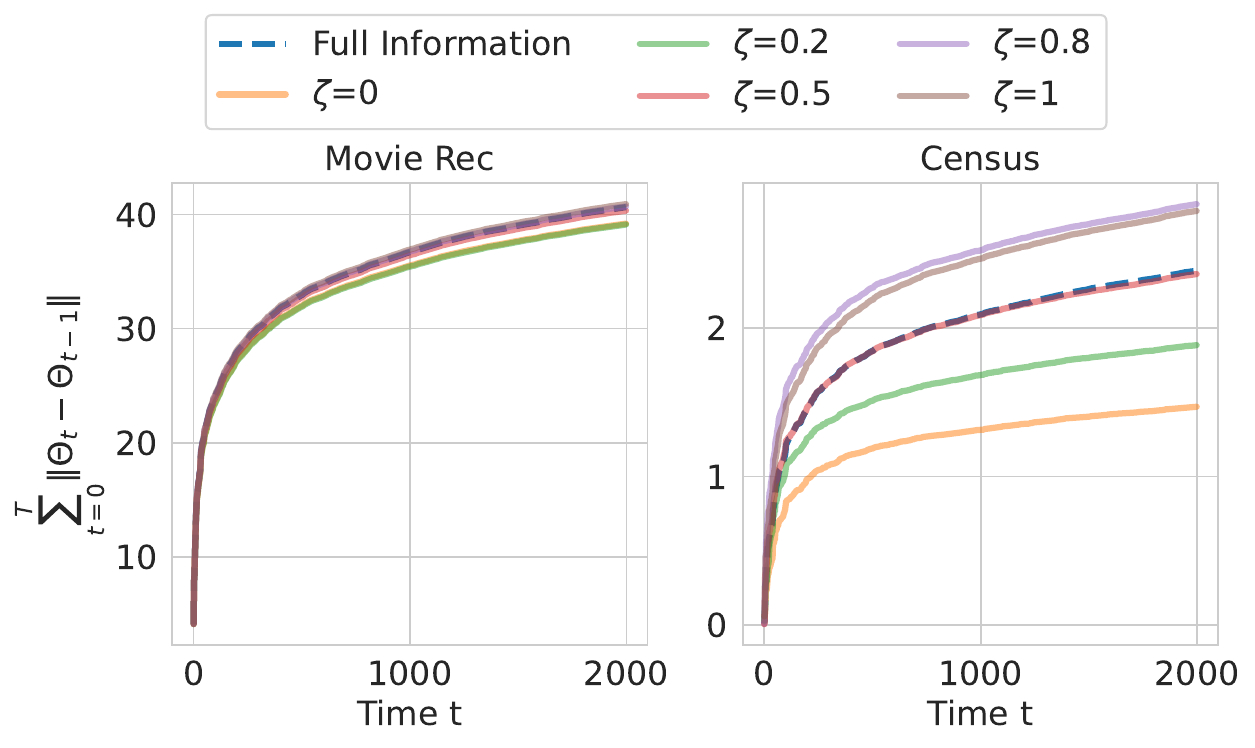}}
\caption{Convergence of iterates $\Theta$ under MSGD or Full Information with $k=3$ services in the movie recommendation (left) and census data (right) tasks. For MSGD, we show results for $\zeta = 0, 0.2, 0.5, 0.8, 1$ respectively.}
\label{fig: convergence of theta}
\end{figure}

We investigate the behavior of MSGD (Algorithm~\ref{alg: simplified single model update}) in the movie recommendation and unemployment prediction settings.
At each time step, we sample a user $x$ at random from the data described above. 
We assign this user to one of $k$ services according to bounded rational with parameters $\zeta$.
Then the selected service updates their parameter with the gradient of the loss on the user's data with step size $\eta^t=\frac{1}{t}$. We evaluate them on a held-out test set. 

We compare MSGD with a \textit{Full Information} algorithm in which user data is shared among services, and at each time step, all the models performs a gradient descent step with this user data, regardless of the user selection.

\textbf{Convergence of Loss Function}\quad
In Figure \ref{fig: convergence of f}, we show the convergence of the learning objective$f(\Theta)$ for settings with users with different levels of rationality. We use a log-log scale, and thus the linear trend in Figure \ref{fig: convergence of f} signifies convergence decreasing polynomially in $t$. Compared with the \textit{Full Information} setting, we find that MSGD performs better when users select services with high rationality, i.e., $\zeta$ is small. 
This is because the full information updates do not allow models to specialize to their own user sup-populations.
However, when $\zeta$ becomes large, MSGD attains higher overall loss than the full information setting. 
This is due to the fact that data sharing allows models to learn faster, since they receive comparably more data.
We conclude that, even with limited data, as long as users select services with sufficient rationality, MSGD can still achieve higher social welfare than when data is shared between all the services (i.e, full information).
\begin{figure}[t]
\centerline{\includegraphics[width=0.5\textwidth]{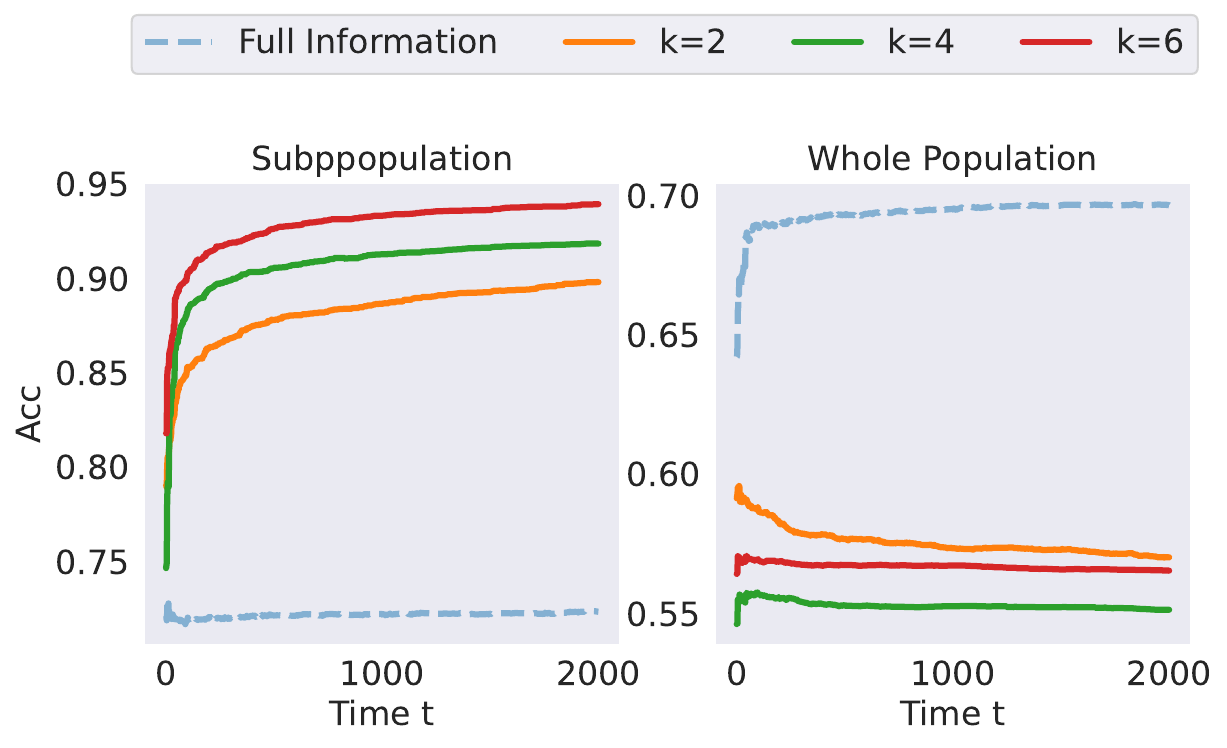}}
\caption{
Accuracy of MSGD or Full Information on the model-specific subpopulation $\mathcal D_i(\Theta)$ (left) and whole population $\mathcal P$ (right) for the ACSEmployment task on census data with perfectly rational users ($\zeta = 0$). For MSGD, we illustrate results of different total number of services $k =2, 4,6$.  }
\label{fig: acc subpopulation v.s. whole population: zeta0}
\end{figure}

\textbf{Convergence of Iterates}\quad In Figure \ref{fig: convergence of theta}, we plot the accumulated parameter update distance $\sum_{t= 0}^T \|\Theta_t-\Theta_{t-1}\|$ to show the convergence of iterates. Since $\Theta$ updates an average of $k=3$ times more often in the full information setting than in MSGD, the accumulated error is $k$ times larger.For fair comparison, we divide the accumulated error of the full information line by the number of services.  We find that, even though the full information setting receives data from a static distribution, it still converges slower compared to MSGD when users select with high rationality.

\textbf{Accuracy over Subpopulation vs. Whole Population} \quad 
So far, we have seen that MSGD is advantageous particularly when users are highly rational,
despite the fact that models have access to less data and act in an uncoordinated manner.
We explore this fact by comparing the induced sub-population performance of model $i$, i.e. the loss on $\mathcal D_i(\Theta)$, to the whole population performance, i.e. the loss on $\mathcal P$ with the ACSEmployment task on census data.
In Figure \ref{fig: acc subpopulation v.s. whole population: zeta0}, we plot the averaged accuracy for $k=2,4,6$ using census data.
Because all services update with the same data in Full Information, the only difference is their initialization, and since they converge to the minimizer of the loss on $\mathcal{P}$, changing $k$ has negligible effect. 
Compared to Full Information, MSGD achieves higher accuracy on induced subpopulations, and this accuracy increases as $k$ increases, as illustrated in
the left graph of Figure \ref{fig: acc subpopulation v.s. whole population: zeta0}. 
However, when evaluated on the whole population $\mathcal{P}$, MSGD performs worse than Full Information, and we even observe accuracy decreasing with more training steps.

In Figure \ref{fig: acc v.s. total number of services}, we illustrate the accuracy over the whole population and the induced subpopulations when the number of services increases with $\zeta=0$ and $\zeta=0.1$ respectively. Notice that when $k$ increases, services updates slower, to ensure services have already converged when calculating the accuracy, for each $k$, we use compute the average accuracy after $T=2000\times k$ total timesteps and plot the average over 3 trials. We observe that when the number of total services $k$ increases, the accuracy over the induced subpopulation increases at the cost of decreasing the accuracy over the whole population. {In practice, the number of service providers depends on an uncoordinated market of services, and choosing $k$ would be like a social planner or regulator intervening on the market to balance the trade-off between accuracy and specialization.}

In the appendix, we investigate the performance of MSGD in additional settings: MSGD under Boltzmann-rational users and minibatch MSGD.
These additional experiments show that despite a lack of theoretical understanding, MSGD also converges for Botzmann-rational users.
They also show that MSGD is able to perform better when minibatching is possible and it is not necessary to operate in a purely streaming setting.
More details can be found in Appendix \ref{app: additional exp}
\begin{figure}[t]
\centerline{\includegraphics[width=0.5\textwidth]{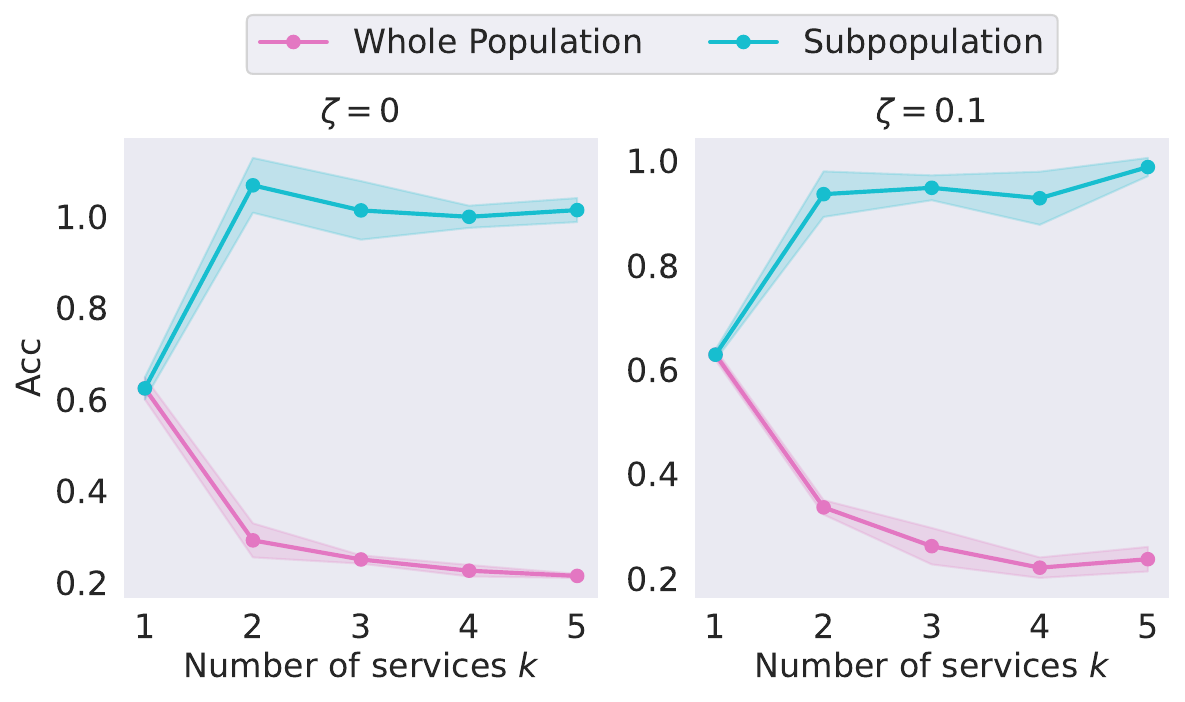}}
\caption{Accuracy of MSGD or Full Information in the census data (right) with fairly rational users and different total number of services $k$. The plot displays mean and standard deviation over three trials. }
\label{fig: acc v.s. total number of services}
\end{figure}

\section{Discussion }

In this paper, 
we consider a setting in which streaming users chose between multiple services, and commit their data to the model of their choosing.
We design a simple, efficient, and intuitive gradient descent algorithm that does not require any coordination between services. We prove that it guarantees convergence to the local optima of the overall objective function, and empirically explore its performance on two different real data settings.

There remain several interesting directions for future work.
One thread comes from considering alternative user behavior models.
{Though we experimentally show that MSGD also converges under the  Boltzmann-rational model, we leave as future work the theoretical analysis.}
Another direction is to consider alternative learning objectives, such as overall population accuracy or market share of each service.
This perspective would motivate greater coordination or explicit competition between the learning-based services,
rather than the simple decentralized updates that we study.

\clearpage

\section*{Acknowledgements}
This work was partly funded by NSF CCF 2312774 and NSF OAC-2311521, and a LinkedIn Research Award.

\section*{Impact Statement} 
This paper presents work whose goal is to advance the field of Machine Learning. There are many potential societal consequences of our work, most of which we do not specifically highlight here. 
We do remark that the MSGD results in specialized models: services optimize for accuracy on the sub-population of users who choose them.
There is a tension between such specialization, which enables more accurate models, and global accuracy, which ensures universal performance over an entire population.
It is important for real world deployments to consider this trade-off carefully in light of fairness, bias, accessibility, etc.
\bibliography{example_paper}
\bibliographystyle{icml2024}

\newpage
\appendix
\onecolumn

\section{Notations, Assumptions and Definitions}
We summarize the notation used in the paper and the proof in Table \ref{notation}. When it is clear from the context, we shall omit the $\Theta$ in the parenthesis and use a simpler notation $X=(X_1, X_2, \cdots, X_k)$, $a=(a_1, \cdots, a_k)$ and $\mathcal{D} = (\mathcal{D}_1, \cdots, \mathcal{D}_k)$. 
\subsection{Notations}
\begin{table}[h]
\vskip 0.15in
\begin{center}
\begin{small}
\begin{tabular}{c|c|c}
\toprule
Notation& Explanation & Value \\
\midrule
$\ell(\cdot, \cdot)$& Personalized loss function & $\ell(\cdot, \cdot)\geq 0$\\
\midrule
$\Theta^t$ & Collection of model parameters of all the services at time step $t$& $\Theta^t=(\theta_1^t, \cdots, \theta_k^t)\in \mathbb{R}^{k\times d}$\\
\midrule
$\theta_i^t$& Model parameter of service provider $i$ at time step $t$& $\theta_i^t\in \mathbb{R}^d$\\
\midrule
$\mathcal{P}$& Underlying data distribution &$\mathcal P\in\Delta(\mathcal X)$\\
\midrule
$x^{t+1}$ & Data arrives time step $t$& $x^{t+1}\sim \mathcal{P}$\\
\midrule 
$\zeta$& The probability that users pick service providers randomly & $\zeta \in [0, 1]$\\
\midrule
$X(\Theta) = (X_1(\Theta), \cdots, X_k(\Theta))$& Data subpopulation partitioning induced by $\Theta$& $\cup_{i\in [k]} X_{i} = \mathcal{X}$\\
\midrule
$a(\Theta) = (a_1(\Theta), \cdots, a_k(\Theta))$& Data subpopulation portion induced by $\Theta$& $\sum_{i=1}^k a_i$ = 1\\
\midrule
$\mathcal{D}(\Theta) = (\mathcal{D}_1(\Theta), \cdots, \mathcal{D}_k(\Theta))$&Subpopulation distribution induced by $\Theta$&$\mathcal{D}_i(\Theta)=\mathcal{P}|_{X_i(\Theta)}$\\
\bottomrule
\end{tabular}
\caption{Notation summary.}\label{notation}
\end{small}
\end{center}
\vskip -0.1in
\end{table}
In the proof, our notations are consistent with Algorithm \ref{alg: general verion}, which is the more verbose version of MSGD. 
\begin{algorithm}[h]
	\caption{Detailed version of MSGD\label{alg: general verion}}
	\begin{algorithmic}
		\STATE {\bfseries Input:} Rationality Parameter $\zeta$; loss function $\ell(\cdot,\cdot)\geq 0$; Initial model parameters $\Theta = (\theta_1, \cdots, \theta_k)$; Learning rate parameter $\{\eta^t\}_{t=1}^{\infty}$ for model that has actual update at each round.
       \FOR{$t=0, 1,2, \cdots, T$ }
\STATE   Receive data point $x^{t+1}\sim \mathcal{P}$,

\STATE\textcolor{lightgray}{User Side: \textbf{User Select a Service Provider:}}\\
\STATE
User rationally
picks the best model $i\in [k]$ where $i=\arg\min_{i\in [k]} \ell(x, \theta_i)$ w.p. $1-\zeta$; Otherwise user randomly picks some $i\in [k]$  w.p. $\zeta$. 
\STATE\textcolor{lightgray}{Learner Side: \textbf{Selected Learner Updates its Model:}}\\
\STATE Let $\eta_i^{t+1}=\eta^{t+1}$ if model $i$ is selected; else $\eta_i^{t+1}=0$
\STATE   Models update through
$\theta_{i}^{t+1}= \theta_{i}^t-\eta_i^{t+1} \cdot\nabla\ell(x^{t+1},\theta_i^t)$; 
        \ENDFOR\\
        \STATE \textbf{Return }
        $\Theta^{T+1} = (\theta_1^{T+1}, \cdots, \theta_k^{T+1})$
	\end{algorithmic}
\end{algorithm}

We also use the notation of filtration, which is given below:
\begin{definition}\label{def: filtration}
Let $(\mathcal{F}^t)_{t=0}^{\infty}$ be a filtration associated with Algorithm \ref{alg: general verion}. In particular, let $\mathcal{F}^0=\sigma(\Theta^0)$ be the $\sigma$-algebra generated by $\Theta^0$, $x^t$ and $\eta^t$ be the tuple where $x^t=(x_1^t, \cdots, x_k^t)$ and $\eta^t = (\eta_1^t, \cdots, \eta_k^t)$.
Let $\sigma$-algebra $\mathcal{F}^t =\sigma (x^t, \eta^t, \mathcal{F}^{t-1})$ contains all random events up to iteration $t$. 
Note that this does not explicitly contain $i^{t}$, the model chosen at time $t$, as this information is included in the definition of $\eta^t$.

In addition, assume
\begin{enumerate}
\item[(1)] If $a_i(\Theta^t)=0$, then $\eta_i^{t+1}=0$ almost surely. 
\item[(2)] $\eta^{t+1}$ and $x^{t+1}$ are conditionally independent given $\mathcal{F}^t$
\item[(3)] $0\leq \eta_i^{t+1}\leq 1$.
\end{enumerate}
\end{definition}

\subsection{Assumptions and Definitions}
In this paper, we use make some rather general assumptions on loss function $\ell(x,\theta)$ (in Assumption \ref{ass: loss function}), such as $L$-Lipschitz (Definition \ref{def: 1}) and $\beta$-smooth (Definition \ref{def: 2}), whose definitions are given below. 

\begin{definition}($L$-Lipschitz) \label{def: 1}
Given the loss function $\ell(\cdot,\cdot): \mathcal{X}\times \mathcal{C}\rightarrow \mathbb{R}$, it is $L$-Lipschitz if for all $x\in \mathcal{X}$ and $\theta, \theta^{'}\in \mathcal{C}$, we have
\begin{equation*}
|\ell(x, \theta)-\ell(x, \theta^{'})|\leq L \cdot \|\theta-\theta^{'}\|.
\end{equation*}
\end{definition}

\begin{definition}($\beta$-smooth)\label{def: 2}
Given the loss function $\ell(\cdot, \cdot):\mathcal{X}\times\mathcal{C}\rightarrow \mathbb{R}$, we say it is $\beta$-smooth if the gradient is $\beta$-Lipschitz, namely, $\forall x\in\mathcal{X}$ and $\theta, \theta^{'}\in \mathcal{C}$,
\begin{equation*}
\|\nabla\ell(x,\theta)-\nabla\ell (x, \theta^{'})\|\leq \beta \cdot \|\theta-\theta^{'}\|
\end{equation*}
\end{definition}
Note that, the above definition is equivalent to 
\begin{equation}
    \ell(x,\theta^{'}) \leq \ell(x, \theta)+\nabla_{\theta}\ell(x,\theta)^T (\theta^{'}-\theta) +\frac{\beta}{2}\|\theta-\theta^{'}\|^2.
\end{equation}

In the proof, we also use martingales.

\begin{definition}(Martingale Sequences)
Define a filtration as an increasing sequence of $\sigma$-fields 
$\emptyset =\mathcal{F}_0\subseteq \mathcal{F}_1\subseteq \cdots \subseteq \mathcal{F}_n$ on some probability space. Let $(X_i)$ be a sequence of random variables such that $X_i$ is measurable w.r.t. $\mathcal{F}_i$, then $(X_i)$ is a martingale w.r.t. $(\mathcal{F})_i$ if $\mathbb{E}[X_i|\mathcal{F}_{i-1}] = X_{i-1}$ for all $i$. 
\end{definition}
\begin{definition}(Martingale Difference Sequences)
We call $(X_i)$ is a martingale difference sequence  w.r.t. $(\mathcal{F})_i$ if $\mathbb{E}[X_i|\mathcal{F}_{i-1}] = 0, \forall i$. 
\end{definition}

\subsection{Background Results}
We also need the following theorem and lemmas for our proof. 
\begin{theorem}\label{thm: martingale convergence}(Martingale convergence theorem \cite{durrett2019probability}) Let $(M^t)_{t=0}^{\infty}$ be a (sub)martingale with 
\begin{equation*}
\sup_{t\in \mathbb{N}}\mathbb{E}[M^t_{+}]<\infty
\end{equation*}
where $M_{+}^t=\max\{0, M^t\}$. Then as $t\rightarrow \infty$, $M^t$ converges a.s. to a limit $M$ with $\mathbb{E}[|M|]<\infty$. 

\end{theorem}
\begin{lemma}(Second Borel-Cantelli Lemma \cite{durrett2019probability}) \label{lemma: Second Borel-Cantelli Lemma}

Let $(\Omega, \mathcal{F}, \mathcal{P})$ be a probability space, let $(\mathcal{F}^t)_{t\geq 0}$ be a filtration with $\mathcal{F}^0=\{0, \Omega\}$ and $(B^t)_{t\geq 0}$ a sequence of events with $B^t\in \mathcal{F}^t$, then,
the event $\{B^t \text{ occurs infinitely often.}\}$ is the same as $\{\sum_{t=0}^{\infty} P(B^t|\mathcal{F}^{t-1})=\infty\}$.
\end{lemma}
\begin{lemma}(Azuma-Hoeffding Inequality)\label{lemma: Azuma-Hoeffding Inequality}
For a sequence of Martingale difference sequence random variable $\{Y_t\}_{t=1}^{\infty}$ w.r.t. filtration $\{\mathcal{F}_t\}_{t=1}^{\infty}$, if we have $a_t\leq Y_t \leq b_t$ for some constant $a_t, b_t, t=1, \cdots, n$, then
\begin{equation}
\text{Pr}(\sum_{t=1}^n Y_t>s)\leq \exp\left(\frac{-2 s^2}{\sum_{t=1}^n (b_t-a_t)^2}\right)
\end{equation}
\end{lemma}

\section{Properties of Learning Objective}\label{app: properties}
\subsection{Decomposition of $f(\Theta)$}\label{app: decomposition}
The learning objective $f(\Theta)$ can be decomposed as follows:
\begin{equation*}
\begin{aligned}
f(\Theta) 
=& \underset{x\sim \mathcal P}{\mathbb{E}}\Big[\sum_{i=1}^k \Big ((1-\zeta) \mathbb 1_i(x,\Theta) +\frac{\zeta}{k}\Big)\ell(x,\theta_i)\Big]\\
=& \sum_{i=1}^k (1-\zeta) \underset{x\sim \mathcal P}{\mathbb{E}}[\mathbb 1_i(x,\Theta)\ell(x,\theta_i)] + \sum_{i=1}^k \frac{\zeta}{k}\underset{x\sim \mathcal P}{\mathbb{E}}[\ell(x,\theta_i)]\\
=& (1-\zeta) \sum_{i=1}^k a_i(\Theta) \cdot \underset{x\sim \mathcal{D}_i(\Theta)}{\mathbb{E}}[\ell(x, \theta_i)] + \frac{\zeta}{k} \cdot\sum_{i=1}^k\underset{x\sim \mathcal{P}}{\mathbb{E}}[\ell(x, \theta_i)]\\
=& (1-\zeta) \cdot f_{\text{PR}}(\Theta) + \zeta \cdot f_{\text{NP}}(\Theta)
\end{aligned}
\end{equation*}
where 
\begin{equation*}
f_{\text{PR}}(\Theta) = \sum_{i=1}^k a_i \cdot \underset{x\sim \mathcal{D}_i}{\mathbb{E}}[\ell(x, \theta_i)]
\end{equation*}
is the learning objective when all users have perfect rationality, while 
\begin{equation*}
f_{\text{NP}}(\Theta) = \frac{1}{k}\sum_{i=1}^k\underset{x\sim \mathcal{P}}{\mathbb{E}}[\ell(x, \theta_i)]
\end{equation*}
is the learning objective when users have no preference and choose randomly over all the models. 
Thus, $f(\Theta)$ can be decomposed as 
a linear combination of $f_{\text{PR}}$ and $f_{\text{NP}}$, controlled by parameter $\zeta$, which captures users' rationality.

\subsection{Example of $f_{\text{PR}}(\Theta)$ being Non-Convex.}\label{app: illustration of f being non-convex}
In Figure \ref{fig: showing f(PR) is non-convex}, we give a simple illustration of $f_{\text{PR}}(\Theta)$ to be non-convex. Here, we show it mathematically. Let $x\sim U(0, 1)$, where $U(0, 1)$ is a uniform distribution over $[0, 1]$, and let the loss function to be $\ell(x, \theta) = (x-\theta)^2$. Let $\Theta = (\theta_1, \theta_2)$ with $\theta_1, \theta_2\in \mathbb{R}$ and $\theta_1<\theta_2$. Then 
\begin{equation*}
\begin{aligned}
f_{\text{PR}}(\Theta)& = \mathbb{E}_{x\sim U(0,1)}[\min\{(x-\theta_1)^2, (x-\theta_2)^2\}]\\
& =\int_{0}^{\frac{\theta_1+\theta_2}{2}}(x-\theta_1)^2 dx +\int_{\frac{\theta_1+\theta_2}{2}}^1 (x-\theta_2)^2 dx\\
& =\frac{(\theta_1+\theta_2)^2(\theta_1-\theta_2)}{4}+\theta_2^2-\theta_2+\frac{1}{3}
\end{aligned}
\end{equation*}
The Hessian of $f(\Theta)$:
\begin{equation*}
\nabla^2 f_{\text{PR}}(\Theta) =
\begin{bmatrix}
  \frac{3\theta_1+\theta_2}{2}& \frac{\theta_1-\theta_2}{2}\\
  \frac{\theta_1-\theta_2}{2}& -\frac{\theta_1+3\theta_2}{2}+2
\end{bmatrix}
\end{equation*}
Without too much constraints on $\Theta$, the Hessian is generally not positive semi-definite, and thus the function is non-convex.

\subsection{Boundedness of $f(\Theta)$. }\label{app: upper boundedness}
Here we prove that $f(\Theta)$ is upper bounded by $f_{\text{NP}}(\Theta)$ and lower bounded by $f_{\text{PR}}(\Theta)$ using Proposition \ref{prop: upper bound function}. 
Specifically, let \begin{equation*}
F(\Theta; \Theta^{'}) = \sum_{i=1}^ka_i(\Theta^{'}) \cdot \underset{x\sim \mathcal{D}_i(\Theta^{'})}{\mathbb{E}}[\ell(x,\theta_i)],
\end{equation*}
which is a family of functions parameterized by $\Theta^{'}$. Then we have $f_{\text{PR}}(\Theta)\leq F(\Theta; \Theta^{'})$ for all $\Theta^{'}$ (Proposition \ref{prop: upper bound function}). 
Let $F(\Theta; \Theta^{'}_i) = \underset{{x\sim \mathcal{P}}}{\mathbb{E}}[\ell(x, \theta_i)]$, namely, 
\begin{align*}
a_j(\Theta^{'}_i)=\begin{cases}
    1, & \text{if } j= i\\
    0, & \text{if } j\neq i, 
\end{cases}
,
~
\mathcal{D}_j(\Theta_i^{'})=\begin{cases}
    \mathcal{P}, & \text{if } j= i\\
    0, & \text{if } j\neq i
\end{cases}
\end{align*}
Then $f_{\text{PR}}(\Theta)\leq \frac{1}{k}\sum_{i=1}^k F(\Theta; \Theta_i^{'})= f_{\text{NP}}(\Theta)$. Thus, $f_{\text{PR}}(\Theta) \leq (1-\zeta) \cdot f_{\text{PR}}(\Theta) + \zeta \cdot f_{\text{NP}}(\Theta)\leq f_{\text{NP}}(\Theta)$, i.e., $f_{\text{PR}}(\Theta) \leq f(\Theta)\leq f_{\text{NP}}(\Theta)$.
\section{Omitted Proofs}
\subsection{Proof of Theorem \ref{thm: main}}
\begin{proof}\label{app: main}
\textbf{(1) Convergence of Objective Function $f(\Theta)$.}

Recall Lemma \ref{lemma: analytic upper bound of f}  gives the following inequality for the updates of $f(\Theta)$:
\begin{equation}\label{eq: 11}
f(\Theta^{t+1}) \leq f(\Theta^t)-A^{t+1}+B^{t+1}-C^{t+1}
\end{equation}

Let $(\mathcal{F}^t)_{t=0}^{\infty}$ be the filtration given in Definition \ref{def: filtration}.

Let $(M^{t})_{t=1}^{\infty} $ be defined by:
\begin{equation*}
M^{t+1} = f(\Theta^{t+1})-\sum_{\tau=0}^t B^{\tau+1}.
\end{equation*}
Then Eq.~\eqref{eq: 11} becomes 
\begin{equation*}
M^{t+1}\leq M^t  -A^{t+1} -C^{t+1}
\end{equation*}
Take the expectation conditioned on $\mathcal{F}^t$:

\begin{align}
\mathbb{E}[M^{t+1}|\mathcal{F}^t]&\leq \mathbb{E}[M^{t}|\mathcal{F}^t]-\mathbb{E}[A^{t+1}|\mathcal{F}^t]-\mathbb{E}[C^{t+1}|\mathcal{F}^t]\notag\\
&\leq \mathbb{E}[M^t|\mathcal{F}^t]-\mathbb{E}[C^{t+1}|\mathcal{F}^t]\label{eq: 13}\\
&=\mathbb{E}[M^t|\mathcal{F}^t]
 \label{eq: 14}\\
&=M^t,\notag
\end{align}
where Eq.~\eqref{eq: 13} is due to $(A^t)_{t=1}^{\infty}$ being non-negative while Eq.~\eqref{eq: 14} is because  $(C^t)_{t=0}^{\infty}$ being a martingale difference sequence (Lemma \ref{lemma: C being martingale}). Thus, $M^t$ is an $\mathcal{F}^t$-supermartingale. Thus, we have showed that $-M^t$ is an $\mathcal{F}^t$-submartingale.

Moreover, since $(B^t)_{t=1}^{\infty}$ is non-negative, we have
\begin{equation*}
\begin{aligned}
-M^t=& -f(\Theta^{t})+\sum_{\tau=1}^{t-1} B^{\tau}\\
\leq& -f(\Theta^t)+\sum_{\tau=1}^{\infty} B^{\tau}\\
\leq& 0 +\sum_{\tau=1}^{\infty} B^{\tau}<\infty,
\end{aligned}
\end{equation*}
where we also used the fact that $f(\Theta^t)\geq 0$ and $\sum_{\tau=1}^{\infty} B^{\tau}<\infty$ (Lemma \ref{lemma: C being martingale}). 
By the martingale convergence theorem (Theorem \ref{thm: martingale convergence}), $(-M^t)_{t=1}^{\infty}$ converges almost surely. Hence, $f(\Theta^t)$ converges almost surely.

\textbf{(2) Convergence of Iterates $\Theta$. }
Based on Lemma \ref{lemma: convergence of iterates}, what we have left to do is to verify our learning rate in Algorithm \ref{alg: simplified single model update} satisfies the two conditions in Lemma \ref{lemma: convergence of iterates}. In order to satisfy Lemma \ref{lemma: convergence of iterates}, we take $t_0$ and $s$ such that $s<\frac{\epsilon\eta_cr}{4L}$ and $t_0\geq \left(\frac{2\ln 6}{s}\eta_c^2\right)\cdot (e^r -1)$.

To begin with, we show that when $\|\nabla_{\theta_i}f(\Theta^{\tau})\|\in [\epsilon, \epsilon_0)$, the first condition $\sum_{\tau\leq t<T(\tau)}\eta^{t+1}<r$ can be satisfied with probability 1.
Define map $T_{r}:\mathbb{N}\rightarrow \mathbb{N}$ so that $T_r(\tau)$ is a unique number s.t.,
\begin{equation}
\sum_{\tau\leq t<T_{r}(\tau)}\frac{1}{t+1}<\frac{r}{\eta_c}\leq \sum_{\tau\leq t\leq T(\tau)} \frac{1}{t+1}
\end{equation}
Then, we have $\sum_{\tau\leq t<T_r(\tau)}\eta^{t+1}=\sum_{\tau\leq t<T_r(\tau)}\frac{\eta_c}{t+1}<r$.
Now we prove that, conditioned on $\mathcal{F}^{\tau} $ and $\|\nabla_{\theta_i}f(\Theta^{\tau})\|\in [\epsilon, \epsilon_0)$,  with some probability, $\sum_{\tau\leq t<T(\tau)}\eta_i^{t+1}>s$ occurs. 

Define $\mu = \sum_{\tau\leq t\leq T(\tau)}\mathbb{E}[\eta_i^{t+1}|\mathcal{F}^t]$ and note that $\eta_i^{t+1}-\mathbb{E}[\eta_i^{t+1}|\mathcal{F}^{t}]$ is a martingale difference sequence with increments
\begin{equation*}
    |(1-\zeta)\cdot a_i^t +\zeta \cdot \frac{1}{k} -1|\cdot \eta^{t+1}\leq \eta^{t+1}=\frac{\eta_c}{t+1},
\end{equation*}
where we have used that $\mathbb{E}[\eta_i^{t+1}|\mathcal{F}^t] = [(1-\zeta)\cdot a_i^t +\zeta \cdot \frac{1}{k}]\cdot \eta^{t+1}$. 
Now, we apply Azuma-Hoeffding's inequality (Lemma \ref{lemma: Azuma-Hoeffding Inequality}), which implies that
\begin{equation*}
\begin{aligned}
&\text{Pr}\left( \sum_{\tau\leq t<T(\tau)}\eta_i^{t+1}>\mu-s\bigg|\mathcal{F}^t, \|\nabla_{\theta_i} f(\Theta^{\tau})\|\in [\epsilon, \epsilon_0)
\right)\\
>&1-\exp\left(\frac{-2s^2}{4\sum_{\tau\leq t
<T(\tau)}\frac{\eta_c^2}{(t+1)^2}}\right)\\
=&1-\exp(-\frac{1}{2}s^2 v)>\frac{5}{6}
\end{aligned}
\end{equation*}
where $v=\left(\sum_{\tau\leq t<T(\tau)}\frac{\eta_c^2}{(t+1)^2}\right)^{-1}$ and we use the fact that 
\begin{equation*}
v=\left(\sum_{\tau\leq t<T(\tau)}\frac{\eta_c^2}{(t+1)^2}\right)^{-1}\overset{(i)}{\geq} \frac{1}{e^{r}-1}\frac{(\tau+1)^2}{(\tau+1)\eta_c^2}\overset{(ii)}{>}\frac{2\ln 6}{s}
\end{equation*}
$(i)$ is because $T(\tau)-\tau \leq (e^\tau -1)\cdot (\tau+1)$ from Corollary \ref{corollary: 1} while $(ii)$ is due to $\tau>t_0\geq (\frac{2\ln 6}{s}\eta_c^2)\cdot (e^r-1)$.  

In the following, we claim and prove that: conditioned on $\mathcal{F}^t$ and $ \|\nabla_{\theta_i}f(\Theta^{\tau})\|\in [\epsilon, \epsilon_0)$, we have $\mu>2s$.
Because 
\begin{align*}
\|\nabla_{\theta_i}f(\Theta^{\tau})\|=&\| (1-\zeta) \cdot a_i^{\tau}(\Theta)\cdot\underset{x\sim \mathcal{D}_i (\Theta)}{\mathbb{E}}[\nabla_{\theta_i}\ell(x,\theta_i)] +  \frac{\zeta }{k}\cdot \underset{x\sim \mathcal{P}}{\mathbb{E}}[\nabla_{\theta_i} \ell(x, \theta_i)]\|\\
\leq &|(1-\zeta)\cdot a^{\tau}_i\cdot L +\frac{\zeta}{k} L|
\end{align*}
Then, we must have $(1-\zeta)\cdot a_i^{\tau}+\frac{\zeta}{k}\geq \frac{\epsilon}{L}$.

Since $a_i(\Theta)$ is locally $L_a$ sensitive (Lemma \ref{lemma: a(Theta) being l_a sensitive}), for $\tau\leq t\leq T(\tau)$, we have $|a_i(\Theta^{t})-a_i(\Theta^{\tau})|\leq L_{a}\|\Theta^{\tau}-\Theta^t\|\leq L_{a}\cdot L\cdot \sum_{\tau\leq t^{'}<t}\eta^{t^{'}}\leq L_{a}Lr$. By setting $r<\frac{\epsilon}{2L^2L_{a}(1-\zeta)}$, we have $(1-\zeta)\cdot a_i^{t}+\frac{\zeta}{k}\geq \frac{\epsilon}{2L}$ for all $\tau\leq t<T(\tau)$.

Then we have 
\begin{equation*}
\begin{aligned}
\mu &= 
\sum_{\tau\leq t< T(\tau)}\mathbb{E}[\eta_i^{t+1}|\mathcal{F}^{\tau}, \|\nabla_{\theta_i}f(\Theta^{\tau})\|\in [\epsilon, \epsilon_0)]\\
&= \sum_{\tau\leq t< T(\tau)}((1-\zeta)\cdot a_i^t+\frac{\zeta}{k})\cdot \frac{\eta_c}{t+1}\\
&\geq \frac{\epsilon \eta_c}{2L}\cdot \sum_{\tau\leq t< T(\tau)}\frac{1}{t+1}\\
&\overset{(i)}{\geq} \frac{\epsilon \eta_c}{2L}\cdot\ln \frac{T(\tau)+1}{\tau+1}\\
& \overset{(ii)}{\geq}\frac{\epsilon \eta_c}{2L}\ln \frac{\tau(\alpha+1) +1}{\tau+1}\\
&\overset{(iii)}{=}\frac{\epsilon \eta_c}{2L}\ln \frac{\tau e^{r}+1}{\tau+1}>2s
\end{aligned}
\end{equation*}
where (i) is from Lemma \ref{lemma: 1}, (ii) is from Corollary \ref{corollary: 1} and (iii) is because $s<\frac{\epsilon \eta_c r}{4L}$.
Therefore, we have proved that \begin{equation*}
\text{Pr}\left( \sum_{\tau\leq t<T(\tau)}\eta_i^{t+1}>\mu-s>s\bigg|\mathcal{F}^t, \|\nabla_{\theta_i} f(\Theta^{\tau})\|\in [\epsilon, \epsilon_0)
\right)>\frac{5}{6}
\end{equation*}

\end{proof}
\subsection{Proof of Lemma \ref{lemma: gradient of learning objective}}\label{app: calculate gradient}
\begin{proof}
As shown in Appendix \ref{app: decomposition}, the objective function $f(\Theta)$ can be decomposed as $f(\Theta) = (1-\zeta)f_{\text{PR}}(\Theta)+\zeta f_{\text{NP}}(\Theta)$, the derivative of $f_{\text{NP}}(\Theta)$ can be easily computed as $\nabla_{\theta_i} f_{\text{NP}}(\Theta) = \frac{1}{k}\underset{x\sim \mathcal{P}}{\mathbb{E}}[\nabla \ell(x, \theta_i)]$. Thus, we will mainly focus on the derivative of $f_{\text{PR}}(\Theta)$, i.e., the learning objective when $\zeta = 0$.

In the following, we will get a closer look at $f_{\text{PR}}$ and then use similar technique from \cite{so2022convergence} by taking the directional derivatives. First, note that $f_{\text{PR}}(\Theta)$, although expressed in terms of $a(\Theta)$ and $\mathcal{D} (\Theta)$ in Eq.~\eqref{eq: 1}, can also be written in terms of indicator functions on $X(\Theta)$, 
\begin{equation}\label{eq: 9}
\begin{aligned}
f_{\text{PR}}(\Theta)& = \sum_{i=1}^k a_i \cdot \underset{x\sim \mathcal{D}_i}{\mathbb{E}}[\ell(x, \theta_i)]\\
&= \underset{x\sim \mathcal{P}}{\mathbb{E}} [\min_{i\in [k]}\ell(x,\theta_i)|\Theta]\\
& =\underset{x\sim\mathcal{P}}{\mathbb{E}}[\sum_{i=1}^k\ell(x,\theta_i) \cdot \mathbb{1}_{X_i}(x) |\Theta],
\end{aligned}
\end{equation}
where $\mathbb{1}_{X_i}(x)$ is the indicator function, 
\begin{equation*}
\mathbb{1}_{X_i}(x)= 
\begin{cases}
    1,& \text{if } x\in X_i(\Theta) \\
    0,              & \text{otherwise}
\end{cases}
\end{equation*}
Fix $\Theta \in \mathbb{R}^{k\times d}$, 
Let $\gamma>0$ be a scalar,  and $v \in \mathbb{R}^{k\times d}$ with $\|v\|=1$. Denote $\tilde{\Theta}=\Theta+\gamma v$ and the subpopulation partition induced by $\tilde{\Theta}$ as $X(\tilde{\Theta}) =(X_1(\tilde{\Theta})\cdots, X_k(\tilde{\Theta}))=(\tilde{X}_1\cdots, \tilde{X}_k)$. Follow Eq.~\eqref{eq: 9}, 
$f_{\text{PR}}(\Theta)=\underset{x\sim\mathcal{P}}{\mathbb{E}}[\sum_{i=1}^k\ell(x,\theta_i) \cdot \mathbb{1}_{X_i}(x) |\Theta]$ while
$f_{\text{PR}}(\tilde{\Theta})$ can be decoupled as 
\begin{equation*}
\begin{aligned}
f_{\text{PR}}(\tilde{\Theta}) 
=&\underset{x\sim\mathcal{P}}{\mathbb{E}}[\sum_{i=1}^k\ell(x,\tilde{\theta}_i) \cdot \mathbb{1}_{\tilde{X}_i}(x) |\tilde{\Theta}]\\
=& \underset{x\sim \mathcal{P}}{\mathbb{E}} [\sum_{i=1}^k\ell(x, \tilde{\theta}_i)\cdot \mathbb{1}_{X_i}(x)+\sum_{i=1}^k\ell(x, \tilde{\theta}_i)\cdot \mathbb{1}_{\tilde{X}_i\backslash X_i}(x)-\sum_{i=1}^k\ell(x, \tilde{\theta}_i)\cdot \mathbb{1}_{X_i\backslash \tilde{X}_i}(x)|\tilde{\Theta},\Theta],
\end{aligned}
\end{equation*}
where we used the fact that $\tilde{X}_i = [X_i \cup (\tilde{X}_i\backslash X_i)]\backslash (X_i \backslash \tilde{X}_i)$.

Now we compute the directional derivative $D_{v} f_{\text{PR}}(\Theta)$ along direction $v$:

\begin{align}
D_v f_{\text{PR}}(\Theta)=&\lim_{\gamma\rightarrow 0}\frac{1}{\gamma}(f_{\text{PR}}(\Theta+\gamma v)-f_{\text{PR}}(\Theta))\notag\\
=&\lim_{\gamma\rightarrow 0}\frac{1}{\gamma}(f_{\text{PR}}(\tilde{\Theta})-f_{\text{PR}}(\Theta))\notag\\
=& \lim_{\gamma\rightarrow 0}\frac{1}{\gamma} \left( \underset{x\sim \mathcal{P}}{\mathbb{E}} \left[\sum_{i=1}^k\ell(x, \tilde{\theta}_i)\cdot \mathbb{1}_{X_i}(x)+\sum_{i=1}^k\ell(x, \tilde{\theta}_i)\cdot \mathbb{1}_{X_i\backslash \tilde{X}_i}(x)-\sum_{i=1}^k\ell(x, \tilde{\theta}_i)\cdot \mathbb{1}_{\tilde{X}_i\backslash X_i}(x)\bigg|\tilde{\Theta}, \Theta\right]\right.\notag\\&\left.-\underset{x\sim\mathcal{P}}{\mathbb{E}}\left[\sum_{i=1}^k\ell(x,\theta_i) \cdot \mathbb{1}_{X_i}(x) \bigg|\Theta\right]\right)\notag\\
=& \lim_{\gamma\rightarrow 0}\frac{1}{\gamma} \left(\underset{x\sim \mathcal{P}}{\mathbb{E}} \left[\sum_{i=1}^k\ell(x, \tilde{\theta}_i)\cdot \mathbb{1}_{X_i}(x)\bigg|\tilde{\Theta},\Theta\right]-\underset{x\sim\mathcal{P}}{\mathbb{E}}\left[\sum_{i=1}^k\ell(x,\theta_i) \cdot \mathbb{1}_{X_i}(x) \bigg|\Theta\right]\right)\notag\\
&+ \lim_{\gamma\rightarrow 0}\frac{1}{\gamma} \left( \underset{x\sim \mathcal{P}}{\mathbb{E}} \left[\sum_{i=1}^k\ell(x, \tilde{\theta}_i)\cdot \mathbb{1}_{X_i\backslash \tilde{X}_i}(x)-\sum_{i=1}^k\ell(x, \tilde{\theta}_i)\cdot \mathbb{1}_{\tilde{X}_i\backslash X_i}(x)\bigg|\tilde{\Theta}, \Theta\right]\right)\label{eq: 4}
 \end{align}
We look at the first two terms and the last two terms of Eq.~\eqref{eq: 4} separately. We show that the first two terms is the directional derivative of a surrogate function that is easy to compute. Then we show that the last two terms is 0. 
To simplify the notation, we introduced a family of surrogate objectives parameterized by  $\Theta^{'}$, 
\begin{equation}\label{eq: 3}
\begin{aligned}
F(\Theta;\Theta^{'}) &= \sum_{i=1}^k a_i(\Theta^{'})\cdot  \underset{x\sim \mathcal{D}_i(\Theta^{'})}{\mathbb{E}}[\ell(x,\theta_i)]\\
&=\underset{x\sim\mathcal{P}}{\mathbb{E}}[\sum_{i=1}^k\ell(x,\theta_i) \cdot \mathbb{1}_{X_i(\Theta^{'})}(x) |\Theta, \Theta^{'}],
\end{aligned}
\end{equation}
Compared to $f_{\text{PR}}$, $F(\Theta; \Theta^{'})$ simply fixes the subpopulation of which the expectation is taken over, making it independent of $\Theta$. Once the subpopulation is fixed, the derivative is easy to compute. 

Note that $\{F(\cdot, \Theta^{'}): \Theta^{'}\in \mathbb{R}^{k\times d}\}$ forms a family of convex, $L$-Lipschitz and $\beta$-smooth functions since it is a sum of convex, $L$-Lipschitz and $\beta$-smooth function $\ell(\cdot, \cdot)$. Moreover, when taking the parameter $\Theta^{'}$ as $\Theta$, we have
\begin{equation*}
F(\Theta; \Theta) = f_{\text{PR}}(\Theta)
\end{equation*}

Then, the first two terms of Eq.~\eqref{eq: 4} can be written as:
\begin{equation}\label{eq: 5}
\begin{aligned}
&\lim_{\gamma\rightarrow 0}\frac{1}{\gamma} \left(\underset{x\sim \mathcal{P}}{\mathbb{E}} \left[\sum_{i=1}^k\ell(x, \tilde{\theta}_i)\cdot \mathbb{1}_{X_i}(x)\bigg|\tilde{\Theta}, \Theta\right]-\underset{x\sim\mathcal{P}}{\mathbb{E}}\left[\sum_{i=1}^k\ell(x,\theta_i) \cdot \mathbb{1}_{X_i}(x) \bigg|\Theta\right]\right)\\
=&\lim_{\gamma\rightarrow 0}\frac{1}{\gamma} \left(\underset{x\sim \mathcal{P}}{\mathbb{E}}\left[\sum_{i=1}^k \left(\ell(x,\tilde{\theta}_i)-\ell(x,\theta_i)\right) \cdot \mathbb{1}_{X_i}(x)\bigg|\tilde{\Theta}, \Theta
\right]
\right)\\
=& \lim_{\gamma\rightarrow 0}\frac{1}{\gamma}\left( F(\tilde{\Theta}; \Theta)-F(\Theta; \Theta)\right) \\
=& \lim_{\gamma\rightarrow 0}\frac{1}{\gamma}\left( F(\Theta+\gamma v; \Theta)-F(\Theta; \Theta)\right) \\
=& D_vF(\Theta; \Theta)\\
\end{aligned}
\end{equation}

where the directional derivative $D_vF(\Theta; \Theta)$ is taken only over the first argument (i.e. the partition $X_i$ is fixed).

Now, let's look at the rest two terms in Eq.~\eqref{eq: 4}. Note that for any point $x\in X_i\backslash \tilde{X}_i$, there must exist some $j\in [k], j\neq i$, such that $x\in \tilde{X}_j$, which are users that prefer model $i$ most compared to other models, but prefers other models (for example, some $j \in [k]$) on the new parameter tuple $\tilde{\Theta}$.

Thus $\sum_{i=1}^k \ell(x, \tilde{\theta}_i)\cdot \mathbb{1}_{X_i\backslash \tilde{X}_i}(x) = \sum_{i,j \in [k], i\neq j}\ell(x, \tilde{\theta}_i)\cdot \mathbb{1}_{X_i\rightarrow \tilde{X}_j}(x)$, where
\begin{equation*}
 \mathbb{1}_{X_i\rightarrow \tilde{X}_j }(x)= 
\begin{cases}
    1,& \text{if }  x \in (X_i \backslash \tilde{X}_i) \cap (\tilde{X}_j\backslash X_j )\\
    0,              & \text{otherwise},
\end{cases}
\end{equation*}
indicating users that prefer model $i$ user parameter tuple $\Theta$ but would choose  model $j$ under parameter tuple $\tilde{\Theta}$. 

Similarly,  for any point $x\in \tilde{X}_i \backslash X_i$, there must exists some $j\in [k], j\neq i$, such that $x\in X_j$, thus, 
$\sum_{i=1}^k \ell(x, \tilde{\theta}_i)\cdot \mathbb{1}_{\tilde{X}_i\backslash X_i}(x) = \sum_{i,j \in [k], i\neq j}\ell(x, \tilde{\theta}_i)\cdot \mathbb{1}_{X_j\rightarrow \tilde{X}_i}(x)$, where
\begin{equation*}
 \mathbb{1}_{X_j\rightarrow \tilde{X}_i }(x)= 
\begin{cases}
    1,& \text{if }  x \in (X_j \backslash \tilde{X}_j) \cap (\tilde{X}_i\backslash X_i )\\
    0,              & \text{otherwise},
\end{cases}
\end{equation*}
indicating users attracted from other services $j\neq i$ to $i$ due to the parameter update from $\Theta$ to $\tilde{\Theta}$.
Therefore, the rest two terms in Eq.~\eqref{eq: 4} can be rewritten as
\begin{equation*}
\begin{aligned}
&\lim_{\gamma\rightarrow 0}\frac{1}{\gamma} \left( \underset{x\sim \mathcal{P}}{\mathbb{E}} \left[\sum_{i=1}^k\ell(x, \tilde{\theta}_i)\cdot \mathbb{1}_{X_i\backslash \tilde{X}_i}(x)-\sum_{i=1}^k\ell(x, \tilde{\theta}_i)\cdot \mathbb{1}_{\tilde{X}_i\backslash X_i}(x)\bigg|\tilde{\Theta},\Theta\right]\right)\\
=& \lim_{\gamma\rightarrow 0}\frac{1}{\gamma}\left( \underset{x\sim \mathcal{P}}{\mathbb{E}} \left[
\sum_{i, j\in [k], i\neq j} \left(\ell(x, \tilde{\theta}_i)- \ell(x, \tilde{\theta}_j)\right) \cdot \mathbb{1}_{X_i\rightarrow \tilde{X}_j}(x)\bigg| \tilde{\Theta}, \Theta
\right]\right)
\end{aligned}
\end{equation*}
For any $x\in  (X_i \backslash \tilde{X}_i) \cap (\tilde{X}_j\backslash X_j )$, i.e., $\mathbb{1}_{X_i\rightarrow \tilde{X}_j}(x)=1$,
according  to the definition, we have 
$\ell(x,\theta_i)\leq \ell(x,\theta_j)$(Under parameter $\Theta$, user $x$ prefers model $i$) and $\ell(x, \tilde{\theta}_j)\leq \ell(x, \tilde{\theta}_i )$ (Under parameter $\tilde{\Theta}$, user $x$ prefers model $j$).

Thus, 
\begin{equation*}
\begin{aligned}
& |\ell(x,\tilde{\theta}_i )-\ell(x,\tilde{\theta}_j)| &\\
=&\ell(x,\tilde{\theta}_i )-\ell(x,\tilde{\theta}_j)& \text{Since }~ \ell(x, \tilde{\theta}_j)\leq \ell(x, \tilde{\theta}_i )\\
=&\ell(x,\theta_i + \gamma v_i)-\ell(x,\theta_j+\gamma v_j)& \\
=& \ell(x, \theta_i) +\gamma v_i \cdot \nabla\ell(x, \theta_i)-\ell(x,\theta_j)-\gamma v_j\cdot \nabla \ell(x,\theta_j)+o(\gamma^2) & \text{(Taylor expansion)  }\\
\leq & \gamma \cdot (v_i\cdot \nabla \ell(x,\theta_i)-v_j\cdot \nabla \ell(x,\theta_j)) +o(\gamma^2)  & \text{Since }~ \ell(x,\theta_i)\leq \ell(x,\theta_j)
\end{aligned}
\end{equation*}

It follows that
\begin{equation}\label{eq: 6}
\begin{aligned}
& \lim_{\gamma\rightarrow 0}\frac{1}{\gamma}\left( \underset{x\sim \mathcal{P}}{\mathbb{E}} \left[
\sum_{i, j\in [k], i\neq j} \left|\ell(x, \tilde{\theta}_i)- \ell(x, \tilde{\theta}_j)\right| \cdot \mathbb{1}_{X_i\rightarrow \tilde{X}_j}(x)\bigg| \tilde{\Theta},\Theta
\right]\right)\\
\leq &  \lim_{\gamma\rightarrow 0}\frac{1}{\gamma}\left( \underset{x\sim \mathcal{P}}{\mathbb{E}} \left[
\sum_{i, j\in [k], i\neq j} \left(\gamma \cdot (v_i\cdot \nabla \ell(x,\theta_i)-v_j\cdot \nabla \ell(x,\theta_j)) +o(\gamma^2)\right) \cdot \mathbb{1}_{X_i\rightarrow \tilde{X}_j}(x)\bigg| \tilde{\Theta},\Theta
\right]\right)\\
=& \lim_{\gamma\rightarrow 0}\left( \underset{x\sim \mathcal{P}}{\mathbb{E}} \left[
\sum_{i, j\in [k], i\neq j} \left(v_i\cdot \nabla \ell(x,\theta_i)-v_j\cdot \nabla \ell(x,\theta_j)\right) \cdot \mathbb{1}_{X_i\rightarrow \tilde{X}_j}(x)\bigg| \tilde{\Theta},\Theta
\right]\right)\\
=&0 ~~~~~~~~~~~~~~((X_i \backslash \tilde{X}_i) \cap (\tilde{X}_j\backslash X_j ) ~\text{decreases to some measure zero set when }~ \gamma \rightarrow 0)
\end{aligned}
\end{equation}

Combining Eq.~\eqref{eq: 3}, Eq.~\eqref{eq: 5} and Eq.~\eqref{eq: 6}, we have
\begin{equation*}
D_v f_{\text{PR}}(\Theta) = D_v F(\Theta; \Theta),
\end{equation*}
which implies that $\nabla_{\Theta}f_{\text{PR}}(\Theta) = \nabla_{\Theta}F(\Theta; \Theta^{'}) $ when $\Theta^{'} = \Theta$.
Note that
\begin{equation*}
\nabla_{\theta_i}F(\Theta; \Theta^{'}) = a_i(\Theta^{'})\cdot\underset{x\sim \mathcal{D}_i(\Theta^{'})}{\mathbb{E}}[\nabla_{\theta_i}\ell(x,\theta_i)]
\end{equation*}
when take the derivative of $F(\Theta; \Theta^{'})$ w.r.t. $\Theta$.

We thus have 
\begin{equation*}
\nabla_{\theta_i} f_{\text{PR}}(\Theta) = a_i(\Theta)\cdot\underset{x\sim \mathcal{D}_i(\Theta)}{\mathbb{E}}[\nabla_{\theta_i}\ell(x,\theta_i)]
\end{equation*}
\end{proof}
\subsection{Proof of Lemma \ref{lemma: analytic upper bound of f}}
\begin{proof}

To simplify the notation, denote model update at each round as $\theta_i^{t+1} = \theta_i^t -\eta_i^{t+1}\cdot \nabla\ell(x^{t+1}, \theta_i^t)$, where $\eta_i^{t+1}\neq 0$ only when model $i$ was selected by the user, i.e.,  $\eta_i^{t+1}=\begin{cases}\eta^{t+1}& \text{, If model } i ~\text{is selected at time } t\\
0&\text{, Otherwise}.
\end{cases}$

\begin{equation*}
\begin{aligned}
f_{\text{PR}}(\Theta^{t+1})\leq & f_{\text{PR}}(\Theta^t) +\langle \nabla f_{\text{PR}}(\Theta^t), \Theta^{t+1}-\Theta^t\rangle +\frac{\beta}{2}\|\Theta^{t+1}-\Theta^t\|^2\\
 \leq& f_{\text{PR}}(\Theta^t)+\sum_{i=1}^k \langle \nabla_{\theta_i} {f_{\text{PR}}}(\Theta^t), \theta_i^{t+1}-\theta_i^t\rangle+\frac{\beta}{2}\sum_{i=1}^k\|\theta_i^{t+1}-\theta_i^t\|^2\\
=& f_{\text{PR}}(\Theta^t)+\sum_{i=1}^k \langle \nabla_{\theta_i} f_{\text{PR}}(\Theta^t), -\eta_i^{t+1} \nabla\ell(x^{t+1}, \theta_i^t)\rangle +\frac{\beta}{2} \sum_{i=1}^k (\eta_i^{t+1})^2 \cdot \|\nabla\ell(x^{t+1}, \theta_i^t)\|^2\\
=& f_{\text{PR}}(\Theta^t)+\sum_{i=1}^k \eta_i^{t+1}\cdot \langle \nabla_{\theta_i} f_{\text{PR}}(\Theta^t), \frac{\nabla_{\theta_i} f_{\text{PR}}(\Theta^t)}{a_i(\Theta^t)}-  \nabla\ell(x^{t+1}, \theta_i^t)- \frac{\nabla_{\theta_i} f_{\text{PR}}(\Theta^t)}{a_i(\Theta^t)}\rangle +\frac{\beta}{2} \sum_{i=1}^k (\eta_i^{t+1})^2 \cdot \|\nabla\ell(x^{t+1}, \theta_i^t)\|^2\\
=& f_{\text{PR}}(\Theta^t) -\underbrace{\sum_{i=1}^k \eta_i^{t+1} \cdot \frac{1}{a_i(\Theta^t)} \cdot \|\nabla_{\theta_i} f_{\text{PR}}(\Theta^t)\|^2 }_{A_{\text{PR}}^{t+1}}\\
&+\underbrace{\frac{\beta}{2} \sum_{i=1}^k (\eta_i^{t+1})^2 \cdot \|\nabla\ell(x^{t+1}, \theta_i^t)\|^2}_{B_{\text{PR}}^{t+1}}
- \underbrace{\sum_{i=1}^k \eta_i^{t+1} \langle \nabla_{\theta_i} f_{\text{PR}}(\Theta^t), \nabla \ell(x^{t+1}, \theta_i^t)-\frac{\nabla_{\theta_i} f_{\text{PR}}(\Theta^t)}{a_i(\Theta^t)}\rangle}_{C_{\text{PR}}^{t+1}}\\
\end{aligned}
\end{equation*}
Similarly, for $f_{\text{NP}}(\Theta)$, we have
\begin{equation*}
\begin{aligned}
f_{\text{NP}}(\Theta^{t+1})\leq & f_{\text{NP}}(\Theta^t) +\langle \nabla f_{\text{NP}}(\Theta^t), \Theta^{t+1}-\Theta^t\rangle +\frac{\beta}{2}\|\Theta^{t+1}-\Theta^t\|^2\\
 \leq& f_{\text{NP}}(\Theta^t)+\sum_{i=1}^k \langle \nabla_{\theta_i} {f_{\text{NP}}}(\Theta^t), \theta_i^{t+1}-\theta_i^t\rangle+\frac{\beta}{2}\sum_{i=1}^k\|\theta_i^{t+1}-\theta_i^t\|^2\\
=& f_{\text{NP}}(\Theta^t)+\sum_{i=1}^k \langle \nabla_{\theta_i} f_{\text{NP}}(\Theta^t), -\eta_i^{t+1} \nabla\ell(x^{t+1}, \theta_i^t)\rangle +\frac{\beta}{2} \sum_{i=1}^k (\eta_i^{t+1})^2 \cdot \|\nabla\ell(x^{t+1}, \theta_i^t)\|^2\\
=& f_{\text{NP}}(\Theta^t)+\sum_{i=1}^k \eta_i^{t+1}\cdot \langle \nabla_{\theta_i} f_{\text{NP}}(\Theta^t), \nabla_{\theta_i} f_{\text{NP}}(\Theta^t)-  \nabla\ell(x^{t+1}, \theta_i^t)- \nabla_{\theta_i} f_{\text{NP}}(\Theta^t)\rangle +\frac{\beta}{2} \sum_{i=1}^k (\eta_i^{t+1})^2 \cdot \|\nabla\ell(x^{t+1}, \theta_i^t)\|^2\\
=& f_{\text{NP}}(\Theta^t) -\underbrace{\sum_{i=1}^k \eta_i^{t+1} \cdot \|\nabla_{\theta_i} f_{\text{NP}}(\Theta^t)\|^2 }_{A_{\text{NP}}^{t+1}}\\
&+\underbrace{\frac{\beta}{2} \sum_{i=1}^k (\eta_i^{t+1})^2 \cdot \|\nabla\ell(x^{t+1}, \theta_i^t)\|^2}_{B_{\text{NP}}^{t+1}}
- \underbrace{\sum_{i=1}^k \eta_i^{t+1} \langle \nabla_{\theta_i} f_{\text{NP}}(\Theta^t), \nabla \ell(x^{t+1}, \theta_i^t)-\nabla_{\theta_i} f_{\text{NP}}(\Theta^t)\rangle}_{C_{\text{NP}}^{t+1}}
\end{aligned}
\end{equation*}
Adding them together, we have 
\begin{equation*}
\begin{aligned}
f(\Theta^{t+1})&=(1-\zeta) f_{\text{PR}}(\Theta^{t+1}) + \zeta f_{\text{NP}}(\Theta^{t+1})\\
&\leq (1-\zeta)\cdot (f_{\text{PR}}(\Theta^{t})-A_{\text{PR}}^{t+1} +B_{\text{PR}}^{t+1}-C_{\text{PR}}^{t+1}) +\zeta \cdot  (f_{\text{NP}}(\Theta^{t})-A_{\text{NP}}^{t+1} +B_{\text{NP}}^{t+1}-C_{\text{NP}}^{t+1})\\
&= f(\Theta^t) -(1-\zeta)A^{t+1}_{\text{PR}}-\zeta A^{t+1}_{\text{NP}} + (1-\zeta)B_{\text{PR}}^{t+1} +\zeta B_{\text{NP}}^{t+1}-(1-\zeta) C_{\text{PR}}^{t+1}-\zeta C_{\text{NP}}^{t+1}
\end{aligned}
\end{equation*}

Letting $i$ denote the model chosen at time $t$. Let $A^{t+1}, B^{t+1}$ and $C^{t+1}$ be 
\begin{align}
A^{t+1}=&(1-\zeta)\frac{\|\nabla_{\theta_i}f_{\text{PR}}(\Theta^t)\|^2}{a_i(\Theta^t)}\notag +\zeta   \eta^{t+1} \|\nabla_{\theta_i}f_{\text{NP}}(\Theta^t)\|^2\notag\\
B^{t+1} =&\frac{\beta}{2}  (\eta^{t+1})^2  \|\nabla\ell(x^{t+1}, \theta_i^t)\|^2\notag\\
C^{t+1} = &(1-\zeta)  \eta^{t+1} \langle \nabla_{\theta_i} f_{\text{PR}}(\Theta^t), \nabla \ell(x^{t+1}, \theta_i^t)-\frac{\nabla_{\theta_i} f_{\text{PR}}(\Theta^t)}{a_i(\Theta^t)}\rangle+\zeta  \eta^{t+1} \langle \nabla_{\theta_i} f_{\text{NP}}(\Theta^t),\nabla \ell(x^{t+1}, \theta_i^t)- \nabla_{\theta_i} f_{\text{NP}}(\Theta^t)\rangle \notag
\end{align}

Then, we have
\begin{equation}\label{eq: 16}
f(\Theta^{t+1})\leq f(\Theta^t) -A^{t+1} +B^{t+1}-C^{t+1}.
\end{equation}

\end{proof}
\subsection{Proof of Lemma \ref{lemma: convergence of iterates}}

\begin{proof}
Since the set of stationary points $\{\nabla f(\Theta) =0\}$ is compact. (Assumption \ref{ass: compact}), and that $\nabla f$ is continuous, there exists $\epsilon_0$, s.t., $\{\|\nabla f\|\leq \epsilon_0\}$ is also compact. 

For any $\epsilon\in (0,\epsilon_0)$, note that $\{\|\nabla_{\theta_i }f\|\leq \frac{\epsilon}{2}\}$ and $\{\|\nabla_{\theta_i}f\|\in [\epsilon,\epsilon_0]\}$ are two closed disjoint compact subsets of $\{\|\nabla_{\theta_i }f\|\leq \epsilon_0\}$ and can be separated by some distance $R_0>0$. Without loss of generality, we assume that $r_0$ satisfy  $r_0\leq \frac{R_0}{L}$. Fix a $r\in (0, r_0)$. Given $(\epsilon, r)$, let $(T, t_0, s,c)$ be chosen so that 
\begin{equation}\label{eq: 12}
\text{Pr}\left(\sum_{\tau\leq t<T(\tau)} \eta^{t+1}<r , \sum_{\tau\leq t< T(\tau)} \eta_i^{t+1}>s\bigg|\mathcal{F}^{\tau}, \|\nabla_{\theta_i} f(\Theta^{\tau})\|\in [\epsilon, \epsilon_0)\right)>c
\end{equation}
holds for any $\tau>t_0$.
 Fixed a $\tau>t_0$, and denote two events 
\begin{equation*}
\begin{aligned}
&\mathcal{O}^{'}_1 = \left\{\sum_{\tau\leq t< T(\tau)} \eta^{t+1}<r\right\},&\mathcal{O}_1 =\left\{\|\Theta^t-\Theta^{\tau}\|<R_0, \forall t\in (\tau,T(\tau)]
\right\}\\
\end{aligned}
\end{equation*}
We first prove that $\mathcal{O}^{'}_1\Rightarrow \mathcal{O}_1$, by applying Lemma  \ref{lemma: learning rate bound iterates}, which implies that the displacement of $\Theta$ can be bounded by the stepsize.  
To see $\mathcal{O}^{'}_1\Rightarrow \mathcal{O}_1$, simply note that, 
for any $t$ such that $\tau< t\leq T(\tau)$, we have
\begin{equation*}
\begin{aligned}
\|\Theta^{t}-\Theta^{\tau}\|&\leq L\cdot \sum_{\tau \leq \tilde{t} < t} \eta^{\tilde{t} +1}& \text{(From Lemma \ref{lemma: learning rate bound iterates})}\\
&< L\cdot \sum_{\tau \leq \tilde{t} < T(\tau)} \eta^{\tilde{t} +1}\\
&<Lr& \text{(If event }\mathcal{O}_1 \text{ occurs.)} \\
&<R_0& \text{(Since } r_0\leq \frac{R_0}{L}\text{ and } r<r_0)
\end{aligned}
\end{equation*}

Fix any $i\in [k]$, denote the following events
\begin{equation*}
\begin{aligned}
&\mathcal{O}_2 =\left\{\sum_{\tau\leq t<T(\tau)} \eta_i^{t+1}>s
\right\}\\
&\mathcal{O}_{3} =\left\{\|\nabla_{\theta_i}f(\Theta^{\tau}\}\|>\frac{\epsilon}{2}\right\}\\
&\mathcal{O}_{4}=\left\{\|\nabla_{\theta_i}f(\Theta^{\tau}\}\|\in[\epsilon,\epsilon_0)\right\}\\
\end{aligned}
\end{equation*}
Then condition Eq.~\eqref{eq: 12} can be rewritten as
\begin{equation*}
\begin{aligned}
\text{Pr}\left( \mathcal{O}^{'}_1 \cap \mathcal{O}_2|\mathcal{F}^{\tau}, \mathcal{O}_4\right)>c,
\end{aligned}
\end{equation*}
which implies that
\begin{equation*}
\begin{aligned}
& \text{Pr}\left(\mathcal{O}_1\cap \mathcal{O}_3 \cap \mathcal{O}_2| \mathcal{F}^{\tau}, \mathcal{O}_4
\right)&\\
=&\text{Pr}\left(\mathcal{O}_1\cap \mathcal{O}_2| \mathcal{F}^{\tau}, \mathcal{O}_4
\right)& (\mathcal{O}_4\Rightarrow \mathcal{O}_3)
\\
\geq &\text{Pr}\left(\mathcal{O}^{'}_1\cap \mathcal{O}_2| \mathcal{F}^{\tau}, \mathcal{O}_4
\right)>c& (\mathcal{O}^{'}_1\Rightarrow \mathcal{O}_1)
\end{aligned}
\end{equation*}

Note that 
$\mathcal{O}_1\cap \mathcal{O}_3$ implies $\|\nabla_{\theta_i} f(\Theta^t)\|>\frac{\epsilon}{2}$ when $\tau \leq t\leq T(\tau)$.

For any $\delta<\frac{s\epsilon^2}{8}$, let $\mathcal{O}_5=\{f(\Theta^{T(\tau)})<f(\Theta^{\tau})-\delta\}$. In the following, we will prove that $\mathcal{O}_5$ occurs if
$\mathcal{O}_1\cap\mathcal{O}_2\cap \mathcal{O}_3$ occurs, i.e., the occurrence of  $\mathcal{O}_1\cap\mathcal{O}_2\cap \mathcal{O}_3$ implies $f(\Theta^t)$ decreases by at least a constant amount on the same interval.

Recall from Lemma \ref{lemma: analytic upper bound of f}, we have 
\begin{equation*}
\begin{aligned}
f(\Theta^{t})\leq & f(\Theta^{t-1})- A^{t}+B^{t} -C^{t}\\
\leq &f(\Theta^{t -2})-A^{t}+B^{t} -C^{t}-A^{t-1}+B^{t-1} -C^{t-1}\\
\leq &\cdots
\leq  f(\Theta^{\tau}) -\sum_{\tau \leq \tilde{t}< t}A^{\tilde{t}+1}+ \sum_{\tau \leq \tilde{t}< t}B^{\tilde{t}+1}- \sum_{\tau \leq \tilde{t}< t}C^{\tilde{t}+1}
\end{aligned}
\end{equation*}
Recall Lemma \ref{lemma: C being martingale} shows that $\sum_{t=0}^{\infty} N^{t+1} = \sum_{t=0}^{\infty}B^{t+1}- \sum_{t=0}^{\infty}C^{t+1}$ convergences almost surely. Therefore, for any $\delta>0$, there exists an $\mathbb{N}$-random variable $M_{\delta}$ such that $\sum_{\tau\leq \tilde{t}<t}N^{\tilde{t}+1}<\delta$ holds for all $\tau>M_{\delta}$.

Then,

\begin{align}
f(\Theta^{T(\tau)})&\leq f(\Theta^{\tau})-\sum_{\tau \leq t< T(\tau)}A^{t+1} +\sum_{\tau \leq t< T(\tau)}N^{t+1}\notag\\
&\leq f(\Theta^{\tau})-\sum_{\tau \leq t< T(\tau)}\sum_{i=1}^k \eta_i^{t+1} \cdot \frac{1}{a_i(\Theta^t)} \cdot \|\nabla f_{\theta_i}(\Theta^t)\|^2 +\delta\notag\\
& \leq f(\Theta^{\tau})-\sum_{\tau \leq t< T(\tau)}\eta_i^{t+1}\cdot \|\nabla f_{\theta_i}(\Theta^t)\|^2 +\delta\label{eq: 17}\\
& \leq f(\Theta^{\tau})-\frac{\epsilon^2}{4}\cdot\sum_{\tau \leq t< T(\tau)}\eta_i^{t+1} +\delta &  (\mathcal{O}_1\cap \mathcal{O}_3 \text{ occurs.)}\notag\\
&\leq f(\Theta^{\tau})-\frac{\epsilon^2s}{4}+\delta & (\mathcal{O}_2 \text{ occurs.)}\notag\\
&<  f(\Theta^{\tau})-\delta & (\text{Set }\delta<\frac{s\epsilon^2}{8})\notag
\end{align}
where in Eq.~\eqref{eq: 17}) we dropped the summation over $j\neq i$, and then dropped $\frac{1}{a_i(\Theta^t)}$ term.

Thus, $\text{Pr}\left(\mathcal{O}_1\cap\mathcal{O}_2\cap\mathcal{O}_3 |\mathcal{F}^{\tau}, \mathcal{O}_4
\right)>c$ indicates that 
\begin{equation*}
\text{Pr}(\mathcal{O}_5|\mathcal{F}^{\tau}, \mathcal{O}_4)>c, 
\end{equation*}
In other words, if $\|f_{\theta_i}(\Theta^{\tau})\|$ is large at iteration $\tau$, then with positive probability, $f(\Theta^{t})$ will decrease by at least $\delta$ from time step $\tau$ to $T(\tau)$, i.e.,
\begin{equation}
\text{Pr}\left(f(\Theta^{T(\tau)})< f(\Theta^{\tau}) -\delta\bigg|\mathcal{F}^{\tau}, \|\nabla_{\theta_i} f(\Theta^{\tau})\|\in [\epsilon, \epsilon_0)\right)>c
\end{equation}

Since $f(\Theta^{t})$ converges almost surely, this decrease of $\delta$ can only happen finite times, and by Borel-Cantelli lemma (Lemma \ref{lemma: Second Borel-Cantelli Lemma}), event $\mathcal{O}_4=\{\|\nabla_{\theta_i}f(\Theta^{\tau})\|\in[\epsilon, \epsilon_0)\}$ must also occur finitely often.

We prove this by contradiction: assume that $\mathcal{O}_4=\{\|\nabla_{\theta_i}f(\Theta^{\tau})\|\in[\epsilon, \epsilon_0)\}$ happens infinitely often, then we can define the infinite sequences of stopping times: $\tau_0=\max\{t_0, M_{\delta}\}$ and $\tau_{j+1} = \inf\{
t\geq T(\tau_j):\|\nabla_{\theta_i}f(\Theta^t)\|\in [\epsilon, \epsilon_0)\}$. Then by Borel-Cantelli lemma (Lemma \ref{lemma: Second Borel-Cantelli Lemma}), $f(\Theta^{t})$ decreases a constant amount of $\delta$ infinitely often, contradicting to the convergence of $f(\Theta^{t})$. To complete the proof, we also have to show that the iterates don't return to the set $\{\|\nabla_{\theta_i}f\|\geq \epsilon_0\}$. Consider the following two cases:

\paragraph{Case 1:} The iterates never leave the set $\{\|\nabla_{\theta_i}f\|\geq \epsilon_0\}$.

\paragraph{Case 2:} The iterates exits and re-enter the set $\{\|\nabla_{\theta_i}f\|\geq \epsilon_0\}$ infinitely often.


Suppose there exists $T$ such that if $\tau>T$, then $\|\nabla_{\theta_i} f(\Theta^{\tau})\|\geq \epsilon_0$, then by Lemma \ref{lemma: convergence of iterates}, we have 
\begin{equation}
    \text{Pr}\left(\sum_{\tau\leq t<T(\tau)}  \eta_i^{t+1}>s\bigg|\mathcal{F}^{\tau}, \tau>\max\{t_0, T\} \right)>c
\end{equation}
By Second Borel-Cantelli Lemma (Lemma \ref{lemma: Second Borel-Cantelli Lemma}), there are infinitely many intervals $\tau\leq t<T(\tau)$ on which $\sum_{\tau\leq t<T(\tau)}\eta_i^{t+1}> s$, so the total sum is infinite almost surely, i.e., $\sum_{t=0}^{\infty} \eta_i^t=\infty$. This leads to unbounded decrease in cost:
\begin{equation}
\lim\inf_{t\rightarrow \infty} f(\Theta^{t})\leq \lim_{t\rightarrow \infty} \left(f(\Theta^{\tau})-\epsilon_0^2 \sum_{\tau \leq \tilde{t}< t} \eta_i^{\tilde{t}+1}  +\delta \right)=-\infty,
\end{equation}
where we used the similar reduction as above:
\begin{align}
f(\Theta^{t})&\leq f(\Theta^{\tau})-\sum_{\tau \leq \tilde{t}< t}A^{\tilde{t}+1} +\sum_{\tau \leq \tilde{t}< t}N^{\tilde{t}+1}\notag\\
&\leq f(\Theta^{\tau})-\sum_{\tau \leq \tilde{t}< t}\sum_{i=1}^k \eta_i^{\tilde{t}+1} \cdot \frac{1}{a_i(\Theta^{\tilde{t}})} \cdot \|\nabla f_{\theta_i}(\Theta^{\tilde{t}})\|^2 +\delta\notag\\
& \leq f(\Theta^{\tau})-\sum_{\tau \leq \tilde{t}< t}\eta_i^{\tilde{t}+1}\cdot \|\nabla f_{\theta_i}(\Theta^{\tilde{t}})\|^2 +\delta\notag\\
& \leq f(\Theta^{\tau})-\epsilon^2_0\cdot\sum_{\tau \leq \tilde{t}< t}\eta_i^{\tilde{t}+1} +\delta \notag
\end{align}
Thus, \textbf{Case 1} is impossible.

As for \textbf{Case 2}, when the learning rate becomes sufficiently small, each time the iterates leave $\{\|\nabla_{\theta_i}f\|\geq \epsilon_0\} $, they must enter $\{\|\nabla_{\theta_i}f\|\in [\epsilon, \epsilon_0)\}$. Thus, the iterates eventually never return to $\{\|\nabla_{\theta_i} f\|\geq \epsilon_0\}$. 

By ruling out both \textbf{Case 1} and \textbf{Case 2}, we have shown that for all $\epsilon>0$, we almost surely have $\|\nabla_{\theta_i} f(\Theta^t)\|>\epsilon$ only finitely often.
\end{proof}

\section{Supporting Lemmas}
\begin{proposition}\label{prop: upper bound function}
Let 
\begin{equation*}
f_{\text{PR}}(\Theta) = \sum_{i=1}^ka_i(\Theta) \cdot \underset{x\sim \mathcal{D}_i(\Theta)}{\mathbb{E}}[\ell(x,\theta_i)]
\end{equation*}
be the perfect rationality objective in Eq.~\eqref{eq: 1},  and \begin{equation*}
F(\Theta; \Theta^{'}) = \sum_{i=1}^ka_i(\Theta^{'}) \cdot \underset{x\sim \mathcal{D}_i(\Theta^{'})}{\mathbb{E}}[\ell(x,\theta_i)],
\end{equation*}
be the surrogate function of introduced in Eq.~\eqref{eq: 3}, then, for all
$\Theta, \Theta^{'}\in \mathbb{R}^{k\times d}$
, we have $f_{\text{PR}}(\Theta)\leq F(\Theta; \Theta^{'})$.
\end{proposition}
\begin{proof}
By the definition of $X(\Theta)$, given $\Theta$, $X(\Theta)$ is the partition that minimizes $f_{\text{PR}}$, namely, 
for any $\Theta^{'} \neq \Theta$, we have
\begin{equation*}
\sum_{i=1}^k a_i(\Theta)\cdot \underset{x\sim \mathcal{D}_i(\Theta)}{\mathbb{E}}[\ell(x,\theta_i)]\leq \sum_{i=1}^k a_i(\Theta^{'})\cdot \underset{x\sim \mathcal{D}_i(\Theta^{'})}{\mathbb{E}}[\ell(x,\theta_i)]
\end{equation*}
To see this, note that, for any data $x\sim \mathcal{P}$, LHS always chooses the best model $\theta$ in $(\theta_1, \cdots, \theta_k)$, which is equivalent to $\underset{x\sim \mathcal{P}}{\mathbb{E}} [\min_{i\in[k]}\ell(x,\theta_i)|\Theta]$, while for RHS, we have 
\begin{equation*}
\begin{aligned}
\text{RHS} = &\underset{x\sim \mathcal{P}}{\mathbb{E}} \left[[\min_{i\in [k]}\ell(x,\theta_i)]\cdot \mathbb{1}_{X\cap X^{'}}(x)|\Theta, \Theta^{'}\right] +\underset{x\sim \mathcal{P}}{\mathbb{E}} \left[\sum_{i,j\in [k], i\neq j}\ell(x, \theta_j)\cdot \mathbb{1}_{X_i\rightarrow X^{'}_j}(x)|\Theta, \Theta^{'}\right]\\
\geq &\underset{x\sim \mathcal{P}}{\mathbb{E}} \left[[\min_{i\in [k]}\ell(x,\theta_i)]\cdot \mathbb{1}_{X\cap X^{'}}(x)|\Theta, \Theta^{'}\right] +\underset{x\sim \mathcal{P}}{\mathbb{E}} \left[\sum_{i,j\in [k], i\neq j}\ell(x, \theta_i)\cdot \mathbb{1}_{X_i\rightarrow X^{'}_j}(x)|\Theta, \Theta^{'}\right]\\
= & \underset{x\sim \mathcal{P}}{\mathbb{E}} [\min_{i\in [k]}\ell(x,\theta_i)|\Theta ]=\text{LHS}
\end{aligned}
\end{equation*}
where \begin{equation*}
 \mathbb{1}_{X\cap X^{'}}(x)= 
\begin{cases}
    1,& \text{if } x\in \cup_{i\in[k]}(X_i(\Theta)\cap X_i(\Theta^{'}))\\
    0,              & \text{otherwise}
\end{cases}
\end{equation*}
\begin{equation*}
 \mathbb{1}_{X_i\rightarrow X^{'}_j}(x)= 
\begin{cases}
    1,& \text{if } x\in X_i(\Theta) \cap X_j(\Theta^{'})\\
    0,              & \text{otherwise}.
\end{cases}
\end{equation*}

$\mathbb{1}_{X\cap X^{'}}(x)$ indicates all the users on the set $\cup_{i\in[k]}(X_i(\Theta)\cap X_i(\Theta^{'}))$, i.e., users that choose the model with the smallest loss, while $\mathbb{1}_{X_i\rightarrow X^{'}_j}(x)$ is the indicator for set $X_i(\Theta) \cap X_j(\Theta^{'})$, i.e., users that should choose model $i$ (which has the smallest loss for them) but incorrectly chooses some other model $j$. Therefore, 
 we have $f_{\text{PR}}(\Theta) = F(\Theta; \Theta)\leq F(\Theta; \Theta^{'})$.

\end{proof}
\begin{lemma}
$f_{\text{NP}}$ is $\beta$-smooth. Moreover, for all $\Theta, \Theta^{+}\in \mathbb{R}^{k\times d}$, we also have 
\begin{equation*}
\begin{aligned}
f_{\text{PR}}(\Theta^{+})\leq f_{\text{PR}}(\Theta) +\langle \nabla f_{\text{PR}}(\Theta), \Theta^{+}-\Theta\rangle +\frac{\beta}{2}\|\Theta^{+}-\Theta\|^2, 
\end{aligned}
\end{equation*}
namely, $f_{\text{PR}}$ is also $\beta$-smooth. 
\end{lemma}
\begin{proof}
Since $f_{\text{NP}}$ is a sum of $\beta$-smooth functions, $f_{\text{NP}}$ 
is also $\beta$-smooth.

Now we prove $f_{\text{PR}}$ is $\beta$-smooth, to show that, we need the Proposition \ref{prop: upper bound function} stating that $f_{\text{PR}}(\Theta)$ is upper bounded by the surrogate functions $\{F(\cdot, \Theta^{'}): \Theta^{'}\in \mathbb{R}^{k\times d}\}$ introduced in Eq.~\eqref{eq: 3}.

Then, for any $\Theta, \Theta^{+}\in \mathbb{R}^{k\times d}$, we have

\begin{align}
f_{\text{PR}}(\Theta^{+})&\leq F(\Theta^{+}; \Theta)\notag\\
& \leq F(\Theta; \Theta) +\nabla F_{\Theta}(\Theta; \Theta)^T(\Theta^{+}-\Theta)+\frac{\beta}{2}\|\Theta^{+}-\Theta\|^2\label{eq: 7}\\
& = f_{\text{PR}
}(\Theta)+\nabla f_{\text{PR}}(\Theta)^T (\Theta^{+}-\Theta) +\frac{\beta}{2}\|\Theta^{+}-\Theta\|^2, \label{eq: 8}
\end{align}

where Eq.~\eqref{eq: 7} used the fact that $F(\cdot; \Theta^{'})$ is $\beta$-smooth $\forall \Theta^{'} \in \mathbb{R}^{k\times d}$, and Eq.(\ref{eq: 8}) follows because $f_{\text{PR}}(\Theta) = F(\Theta; \Theta)$ and $\nabla f_{\text{PR}}(\Theta) = \nabla_{\Theta} F(\Theta; \Theta)$. Thus, $f_{\text{PR}}(\Theta)$ is $\beta$-smooth.

\end{proof}

\begin{lemma}\label{lemma: C being martingale}
Let $(F^t)_{t=0}^{\infty}$ be a filtration given by Definition \ref{def: filtration}.
Let 
$B^{t+1} = \mathbb{E}_s[B_s^{t+1}] = (1-\zeta) B_{\text{PR}}^{t+1} +\zeta  B_{\text{NP}}^{t+1}$, and $C^{t+1} = \mathbb{E}_s[C_s^{t+1}]=(1-\zeta) C_{\text{PR}}^{t+1} +\zeta  C_{\text{NP}}^{t+1}$, where 
\begin{equation*}
B_s^{t+1} =\begin{cases}
B_{\text{PR}}^{t+1}& \text{w.p.}~ 1-\zeta\\
B_{\text{NP}}^{t+1}& \text{w.p.}~ \zeta\\
\end{cases}, 
C_s^{t+1} =\begin{cases}
C_{\text{PR}}^{t+1}& \text{w.p.}~ 1-\zeta\\
C_{\text{NP}}^{t+1}& \text{w.p.}~ \zeta\\
\end{cases}.
\end{equation*}
with
\begin{align*}
B_{\text{PR}}^{t+1} =\frac{\beta}{2} \sum_{i=1}^k (\eta_i^{t+1})^2 \cdot \|\nabla\ell(x^{t+1}, \theta_i^t)\|^2, ~& B_{\text{NP}}^{t+1} =\frac{\beta}{2} \sum_{i=1}^k (\eta_i^{t+1})^2 \cdot \|\nabla\ell(x^{t+1}, \theta_i^t)\|^2\\
C_{\text{PR}}^{t+1} =\sum_{i=1}^k \eta_i^{t+1} \langle \nabla_{\theta_i} f_{\text{PR}}(\Theta^t), \nabla \ell(x^{t+1}, \theta_i^t)-\frac{\nabla_{\theta_i} f_{\text{PR}}(\Theta^t)}{a_i(\Theta^t)}\rangle, ~& C_{\text{NP}}^{t+1} =\sum_{i=1}^k \eta_i^{t+1} \langle \nabla_{\theta_i} f_{\text{NP}}(\Theta^t), \nabla \ell(x^{t+1}, \theta_i^t)-\nabla_{\theta_i} f_{\text{NP}}(\Theta^t)\rangle
\end{align*}
and the expectation is taken over users random selection over services. 
Let $N^{t+1}=B^{t+1}-C^{t+1}$, 
then 
\begin{enumerate}
    \item $(C^t)_{t=0}^{\infty}$ is $\mathcal{F}^t$-martingale difference sequences, and $\mathbb{E}[C^{t+1}|\mathcal{F}^t]=0$.
    \item Suppose that  $\sum_{i\in[k]}\sum_{t=1}^{\infty} (\eta_i^t)^2< \infty$ a.s. Then the series $\sum_{t=1}^{\infty} B^t<\infty$ converges almost surely. 
    \item The series $\sum_{t=1}^{\infty}N^t = \sum_{t=1}^{\infty}B^t -\sum_{t=1}^{\infty}C^t<\infty$ converges almost surely.
\end{enumerate}

\end{lemma}
\begin{proof}

\textbf{(1) Proof of $(C^t)_{t=1}^{\infty}$ being martingale difference sequences. }

Take the expectation of $\tilde{C}^{t+1}$ conditioned on $\mathcal{F}^{t}$:
\begin{equation*}
\begin{aligned}
\mathbb{E}[C_s^{t+1}|\mathcal{F}^t] =&(1-\zeta) \cdot \left[\sum_{i=1}^k \langle \nabla_{\theta_i} f_{\text{PR}}(\Theta^t),\mathbb{E}[\nabla \ell(x^{t+1}, \theta_i^t)\cdot \eta_i^{t+1}|\mathcal{F}^t] - \frac{\nabla_{\theta_i} f_{\text{PR}}(\Theta^t)}{a_i(\Theta^t)} \cdot\mathbb{E}[\eta_i^{t+1}|\mathcal{F}^t] \rangle\bigg|\text{User choose with PR}\right] \\
&+\zeta \cdot \left[\sum_{i=1}^k \langle \nabla_{\theta_i} f_{\text{NP}}(\Theta^t), \mathbb{E}[\nabla \ell(x^{t+1}, \theta_i^t)\cdot \eta_i^{t+1}|\mathcal{F}^t] - \nabla_{\theta_i} f_{\text{NP}}(\Theta^t)\cdot \mathbb{E}[\eta_i^{t+1}|\mathcal{F}^t]\rangle\bigg| \text{User choose with NP}\right]
\end{aligned}
\end{equation*}

Conditioned on user choosing with prefect rationality, we have $\mathbb{E}[\eta_i^{t+1}|\mathcal{F}^t] = a_i^t\cdot \eta^{t+1}$. Thus
\begin{equation*}
\begin{aligned}
&\mathbb{E}[\nabla \ell(x^{t+1}, \theta_i^t)\cdot \eta_i^{t+1}|\mathcal{F}^t]= a_i(\Theta) \cdot \underset{x\sim \mathcal{D}_i(\Theta)}{\mathbb{E}}[\nabla_{\theta_i}\ell(x,\theta_i)]\cdot \eta^{t+1}\\
&\frac{\nabla_{\theta_i} f_{\text{PR}}(\Theta)}{a_i(\Theta^t)} \cdot\mathbb{E}[\eta_i^{t+1}|\mathcal{F}^t]= \nabla_{\theta_i} f_{\text{PR}}(\Theta) \cdot\eta^{t+1}= a_i(\Theta)\cdot\underset{x\sim \mathcal{D}_i(\Theta)}{\mathbb{E}}[\nabla_{\theta_i}\ell(x,\theta_i)]\cdot \eta^{t+1}
\end{aligned}
\end{equation*}

As a result, we have shown that 
\begin{equation*}
\frac{\nabla_{\theta_i} f_{\text{PR}}(\Theta^t)}{a_i(\Theta^t)} \cdot\mathbb{E}[\eta_i^{t+1}|\mathcal{F}^t] =\mathbb{E}[\nabla \ell(x^{t+1}, \theta_i^t)\cdot \eta_i^{t+1}|\mathcal{F}^t] 
\end{equation*}

Conditioned on user choosing randomly among models, we have $\mathbb{E}[\eta_i^{t+1}|\mathcal{F}^t] = \frac{1}{k}\cdot \frac{\eta_c}{t+1}$, which leads to
\begin{equation}
\begin{aligned}
&\nabla_{\theta_i} f_{\text{NP}}(\Theta^t)\cdot \mathbb{E}[\eta_i^{t+1}|\mathcal{F}^t]= \eta^{t+1}\cdot \frac{1}{k} \underset{x\sim \mathcal{P}}{\mathbb{E}} [\nabla \ell(x, \theta_i)]\\
&
\mathbb{E}[\nabla \ell(x^{t+1}, \theta_i^t)\cdot \eta_i^{t+1}|\mathcal{F}^t]=\eta^{t+1} \cdot \frac{1}{k} \underset{x\sim \mathcal{P}}{\mathbb{E}} [\nabla \ell(x, \theta_i)], 
\end{aligned}
\end{equation}
Thus, we also get 
\begin{equation}
\nabla_{\theta_i} f_{\text{NP}}(\Theta^t)\cdot \mathbb{E}[\eta_i^{t+1}|\mathcal{F}^t]=
\mathbb{E}[\nabla \ell(x^{t+1}, \theta_i^t)\cdot \eta_i^{t+1}|\mathcal{F}^t]
\end{equation}

In the end, we have $\mathbb{E}[C_s^{t+1}|\mathcal{F}^t] = 0$, i.e., $\mathbb{E}[\mathbb{E}_s[C_s^{t+1}|\mathcal{F}^t]]= 0$. And thus, we have $\mathbb{E}[C^{t+1}|\mathcal{F}^t] = 0$.

\textbf{(2) Prove that $\sum_{t=1}^{\infty} B^t< \infty$ converges. }

Since $\sum_{i\in[k]}\sum_{t=1}^{\infty} (\eta_i^t)^2 <\infty$, we have
\begin{equation*}
\begin{aligned}
\sum_{t=1}^{\infty} B_{\text{PR}}^t=
& \frac{\beta}{2}\sum_{t=1}^{\infty} \sum_{i=1}^k (\eta_i^t)^2\cdot \|\nabla \ell(x_i^t, \theta_i^{t-1})\|^2\\
\leq & \frac{\beta L^2}{2} \sum_{t=1}^{\infty} \sum_{i=1}^k (\eta_i^t)^2<\infty
\end{aligned}
\end{equation*}
Similarly, $\sum_{t=1}^{\infty}B^t_{\text{NP}}<\infty$. 
Thus, $\sum_{t=1}^{\infty} B^t<\infty$ converges almost surely.

\textbf{(3) Prove that $\sum_{t=1}^{\infty} N^t< \infty$ converges.}

Let $H^t = \sum_{\tau=1}^t C^{\tau}$ and let $H_{+}^t = \max\{0, H^t\}$, since the terms in a martingale difference sequence are orthogonal, we have for all $t\in \mathbb{N}$:
\begin{equation}\label{eq: 10}
\begin{aligned}
\mathbb{E}[H_{+}^t]
&\leq \sqrt{\mathbb{E}[(H^t)^2]}\\
&= \sqrt{\mathbb{E}\left[\left(\sum_{\tau=1}^t C^{\tau}\right)^2\right]}\\
&=\sqrt{\sum_{\tau=1}^t \mathbb{E}[(C^{\tau})^2]}
\end{aligned}
\end{equation}
Note that
\begin{equation*}
\begin{aligned}
&|\langle \nabla_{\theta_i} f_{\text{PR}}(\Theta^t),\nabla \ell(x^{t+1}, \theta_i^t)-  \frac{\nabla_{\theta_i} f_{\text{PR}}(\Theta^t)}{a_i(\Theta^t)}\rangle|\\
\leq& \| \nabla_{\theta_i} f_{\text{PR}}(\Theta^t)\|\cdot \|\nabla \ell(x^{t+1}, \theta_i^t) - \frac{\nabla_{\theta_i} f_{\text{PR}}(\Theta^t)}{a_i(\Theta^t)}\|\\
\leq & 2L^2.
\end{aligned}
\end{equation*}
Similarly, we can also get 
\begin{equation*}
\begin{aligned}
|\langle \nabla_{\theta_i} f_{\text{NP}}(\Theta^t), \nabla \ell(x^{t+1}, \theta_i^t) - \nabla_{\theta_i} f_{\text{NP}}(\Theta^t)\rangle |
\leq & 2L^2.
\end{aligned}
\end{equation*}

We thus have
\begin{equation*}
\begin{aligned}
\sum_{t=1}^{\infty}\mathbb{E}[(C^t)^2]&\leq \sum_{t=1}^{\infty}\mathbb{E}[(\sum_{i=1}^k \eta_i^{t}\cdot 2 L^2)^2]\\
& \leq 4 L^4 \cdot \sum_{t=1}^{\infty} \sum_{i=1}^k (\eta_i^t)^2<\infty,
\end{aligned}
\end{equation*}

Combining Eq.~\eqref{eq: 10}, it implies that $\sup_{t\in \mathbb{N}}\mathbb{E}[H^t_{+}]<\infty$.
According to martingale convergence theorem (Theorem \ref{thm: martingale convergence}), as $t\rightarrow 0$, $H^t$ converges, and thus $\sum_{t=1}^{\infty} C^t<\infty$ converges.

Moreover, due to the convergence of series $\sum_{t=1}^{\infty} B^t<\infty$,  the series $\sum_{t=1}^{\infty}N^t = \sum_{t=1}^{\infty}B^t -\sum_{t=1}^{\infty}C^t<\infty$ converges almost surely.

\end{proof}

\begin{lemma}\label{lemma: learning rate bound iterates}
For all $t<t^{'}$, the displacement between  $\Theta^{t}$ and $\Theta^{t^{'}}$ satisfies
\begin{equation*}
\|\Theta^{t^{'}}-\Theta^{t}\|  \leq L \cdot  \sum_{t+1\leq \tau \leq t^{'}} \eta^{\tau}
\end{equation*}
\end{lemma}
\begin{proof}
The displacement of $\Theta^{t}$ and $\Theta^{t^{'}}$ can be bounded as 
\begin{equation*}
\begin{aligned}
\|\Theta^{t^{'}}-\Theta^{t}\|&\leq \sum_{i=1}^k \|\theta_i^{t^{'}}-\theta_i^{t}\|& ~\text{ (Minkowski’s inequality)}\\
&\leq  \sum_{i=1}^k \|\theta_i^{t^{'}}- \theta_i^{t^{'}-1} + \theta_i^{t^{'}-1}-\cdots + \theta_i^{t+1}-\theta_i^{t}\|&\\
& \leq \sum_{i=1}^k \sum_{t+1\leq \tau \leq t^{'}}\|\theta_i^{\tau}-\theta_i^{
\tau -1
}\|& \\
& \leq  \sum_{i=1}^k \sum_{t+1\leq \tau \leq t^{'}}\|\theta_i^{\tau-1}-\eta_i^{\tau}\nabla \ell(x^{\tau}, \theta_i^{\tau-1}) -\theta_i^{
\tau -1
}\|& \\
&= \sum_{i=1}^k \sum_{t+1\leq \tau \leq t^{'}}\|\eta_i^{\tau} \nabla \ell(x^{\tau},\theta_i^{\tau-1})\|& \\
& \leq L \cdot \sum_{i=1}^k \sum_{t+1\leq \tau \leq t^{'}} \eta_i^{\tau} & (\text{$L$-Lipschitz of} ~\ell)\\
& = L \cdot \sum_{t+1\leq \tau \leq t^{'}} \eta^{\tau}
\end{aligned}
\end{equation*}

\end{proof}

\begin{lemma}\label{lemma: 1}
Let $1<\tau<\tau^{'}$ where $\tau, \tau^{'}\in \mathbb{N}$, then
\begin{equation}
\ln \frac{\tau^{'}+1}{\tau+1}\leq \sum_{\tau \leq t< \tau^{'}} \frac{1}{t+1}\leq \ln \frac{\tau^{'}}{\tau}
\end{equation}
\end{lemma}
\begin{proof}
Simply note that
\begin{equation*}
    \ln \frac{\tau^{'}+1}{\tau+1}=\int_{\tau}^{\tau^{'}}\frac{1}{x+1} dx \leq \sum_{\tau \leq t <\tau^{'}}\frac{1}{t+1}\leq \int_{\tau}^{\tau^{'}}\frac{1}{x} dx =\ln \frac{\tau^{'}}{\tau}
\end{equation*}
\end{proof}

\begin{corollary}\label{corollary: 1}
    Let $r>0$ and $\alpha= e^{r} -1$, set $T\vcentcolon= T_r$, then
    \begin{equation}
    \alpha \tau \leq T(\tau)-\tau \leq \alpha (\tau+1)
    \end{equation}
\end{corollary}
\begin{proof}
Replacing $\tau^{'}$ to be $T(\tau)$ in Lemma \ref{lemma: 1}, we get
\begin{equation}
\begin{aligned}
&\ln \frac{T(\tau) +1}{\tau+1}\leq \sum_{\tau\leq t<T(\tau)}\frac{1}{t+1}< r \leq \sum_{\tau \leq t<T(\tau) +1}\frac{1}{t+1}\leq \ln \frac{T(\tau)}{\tau}\\
\Rightarrow&\frac{T(\tau) +1}{\tau+1}-1\leq e^r-1\leq \frac{T(\tau)}{\tau}-1\\
\Rightarrow&\frac{T(\tau)-\tau}{\tau+1}\leq e^r -1\leq \frac{T(\tau) -\tau}{\tau}\\
\Rightarrow&T(\tau)-\tau\leq (e^r-1)\cdot(\tau-1)=\alpha \cdot(\tau+1)\\
& T(\tau) -\tau \geq (e^{r}-1)\cdot \tau =\alpha \tau\\
\end{aligned}
\end{equation}
We have thus proved $ \alpha\tau \leq T(\tau)-\tau\leq \alpha (\tau+1)$.
\end{proof}
\begin{lemma}(Local Lipschitz of $a_i(\Theta)$)\label{lemma: a(Theta) being l_a sensitive}
Let $p$ be a continuous density function supported in the closed ball $B(0, R)$. 
Then under Assumption~\ref{ass: assumptions for proving a being Lipschitz}, $a_i(\Theta)$ is locally Lipschitz. 
\end{lemma}
\begin{proof}
First, we prove a simple case where $k=2$ and we only have two services $i=1, j=2$. 
Given two sets  of model parameters $\Theta,  \Theta^{'}$, the difference of the induced portion of model $i$ is
\begin{equation*}
a_i(\Theta)-a_i(\Theta^{'})= \int_{X_i(\Theta)\backslash X_i(\Theta^{'})}p(x) dx -\int_{X_i(\Theta^{'})\backslash X_i(\Theta) }p(x) dx
\end{equation*}

Since $p$ is continuous on the closed ball $B(0, R)$, it attains maximum $p_{\text{max}}=\sup p(x)<\infty$. Then
\begin{equation*}
|a_i(\Theta)-a_i(\Theta^{'})|\leq p_{\text{max}} \cdot \left( \lambda(X_i(\Theta)\backslash X_i(\Theta^{'}))+\lambda (X_i(\Theta^{'})\backslash X_i(\Theta))
\right)
\end{equation*}
where $\lambda$ is the Lebesgue measure. 

Let's look at $X_1(\Theta^{'})\backslash X_1(\Theta)$, which are users move from model 2 to model 1 when $\Theta$ is perturbed to $\Theta^{'}$. For any point $x$ in $X_1(\Theta^{'})\backslash X_1(\Theta)$, we have
\begin{equation*}
\ell(x, \theta_1^{'})<\ell(x, \theta_2^{'}), \ell(x, \theta_2)<\ell(x, \theta_1)
\end{equation*}
Thus,
\begin{equation*}
\begin{aligned}
&\ell(x,\theta_2^{'})-\ell(x, \theta_1^{'})\\
=& \ell(x,\theta_2^{'})-\ell(x, \theta_2)+\ell(x, \theta_2)-\ell(x, \theta_1)+\ell(x, \theta_1)-\ell(x, \theta^{'}_1)\\
\leq & L\cdot \|\theta_2^{'}-\theta_2\|+L\cdot \|\theta_1-\theta_1^{'}\|+\ell(x, \theta_2)-\ell(x, \theta_1)
\end{aligned}
\end{equation*}
Therefore, we have
\begin{align*}
0<\ell(x, \theta_1)-\ell(x, \theta_2)<L\cdot \|\theta_2^{'}-\theta_2\|+L\cdot \|\theta_1-\theta_1^{'}\|\leq 2L\|\Theta-\Theta^{'}\|
\end{align*}
Let $S = \{x: |\ell(x, \theta_1)-\ell(x, \theta_2)|\leq 2L\|\Theta-\Theta^{'}\|\}$, then, from our assumption, we have $\lambda(S)\leq 2L\|\Theta-\Theta^{'}\|$ by letting $d=2L\|\Theta-\Theta^{'}\|$. Since $\lambda(S)\geq \lambda(X_1(\Theta^{'})\backslash X_1(\Theta))$, we have $\lambda(X_1(\Theta^{'})\backslash X_1(\Theta))\leq 2L \|\Theta-\Theta^{'}\|$. Similarly, we have $\lambda(X_1(\Theta)\backslash X_1(\Theta^{'})\leq 2L \|\Theta-\Theta^{'}\|$.
Thus, 
\begin{equation*}
|a_i(\Theta)-a_i(\Theta^{'})|\leq 4 p_{\text{max}}L\|\Theta-\Theta^{'}\|
\end{equation*}
Now that we have prove the lemma for $k=2$, to extend it to general $k$, simply note that $X_j(\Theta^{'})\backslash X_j(\Theta) =\cup_{l\neq j} (X_j(\Theta^{'})\cap X_l(\Theta)   )$, and that 
$\lambda(X_j(\Theta^{'})\backslash X_j(\Theta) )\leq \sum_{l\neq j} \lambda(X_j(\Theta^{'})\cap X_l(\Theta)   ).$
\end{proof}
\section{Additional Experiments} \label{app: additional exp}
\paragraph{More experiments of MSGD for ACSEmployment task on census data.}
As a supplement to Figure \ref{fig: acc subpopulation v.s. whole population: zeta0}, we compare the accuracy of MSGD and full information on the model specific subpopulation and the whole population  when $\zeta =0.1$ (Figure \ref{fig: acc subpopulation v.s. whole population: zeta0.1}) as well as the losses when $\zeta=0$ and $0.1$ (Figure \ref{fig: vs full information on loss}).
While Figure \ref{fig: acc subpopulation v.s. whole population: zeta0.1} shows a similar trend as of Figure \ref{fig: acc subpopulation v.s. whole population: zeta0}, we find that even evaluated on subpopulation, the increased the number of services from $k=2$ to $k=4$ decreases the accuracy on subpopulation. A plausible reason is that, besides the changes initial landscape due to the added service, the increase in $k$ also reduces the average number of data each service provider receives, causing the relationship of the number of services and the accuracy over subpopulation less observable. 
\begin{figure}
\centerline{\includegraphics[width=0.5\textwidth]{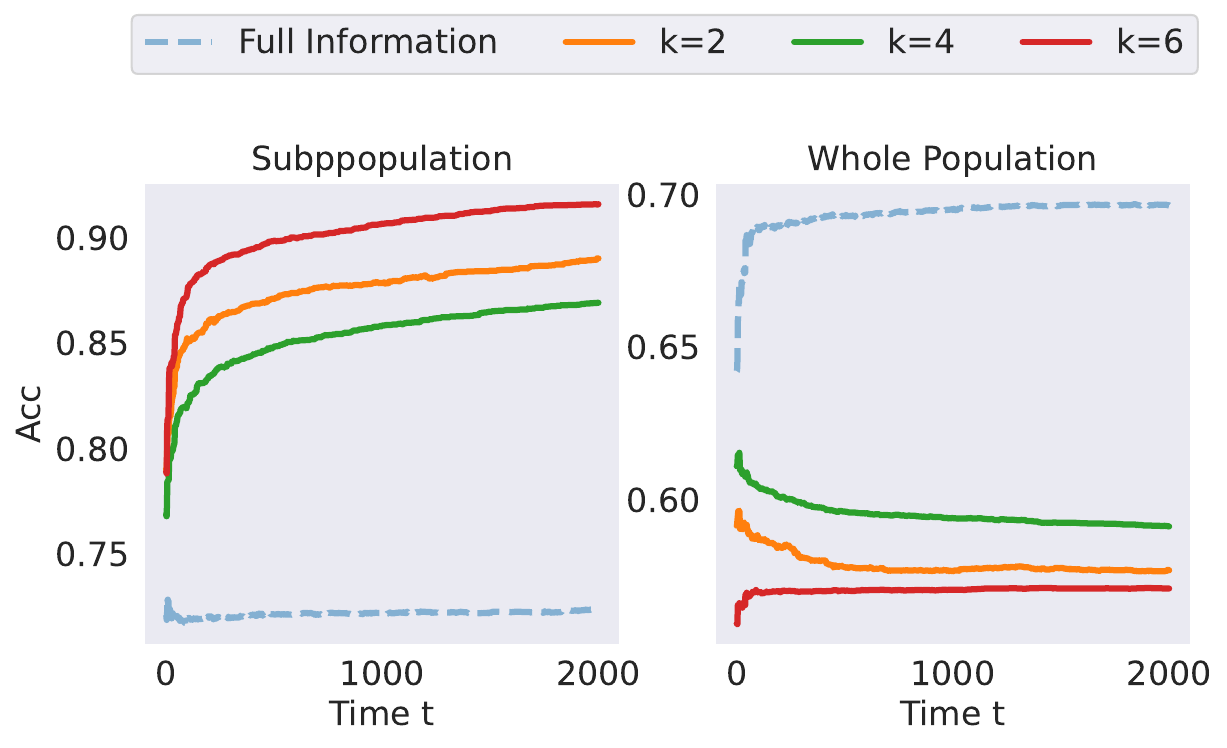}}
\caption{
Accuracy of MSGD or Full Information on the model-specific subpopulation $\mathcal D_i(\Theta)$ (left) and whole population $\mathcal P$ (right) for the ACSEmployment task on census data with perfectly rational users ($\zeta = 0.1$). For MSGD, we illustrate results of different total number of services $k =2, 4,6$.  }
\label{fig: acc subpopulation v.s. whole population: zeta0.1}
\end{figure}

Since services are trained with loss functions, we compare the overall loss of MSGD and Full information  over both subpopulation and whole population in Figure \ref{fig: vs full information on loss}. We found that, the loss of full information decreases both in terms of subpopulation and the whole population. Meanwhile, even though MSGD decreases faster than full information in subpopulation loss, its loss over the whole population increases, which means they becomes worse and worse for general population. (Similar to what we see about the accuracy).
\begin{figure}[ht]
\begin{subfigure}{.5\textwidth}
  \centering
  \includegraphics[width=1\linewidth]{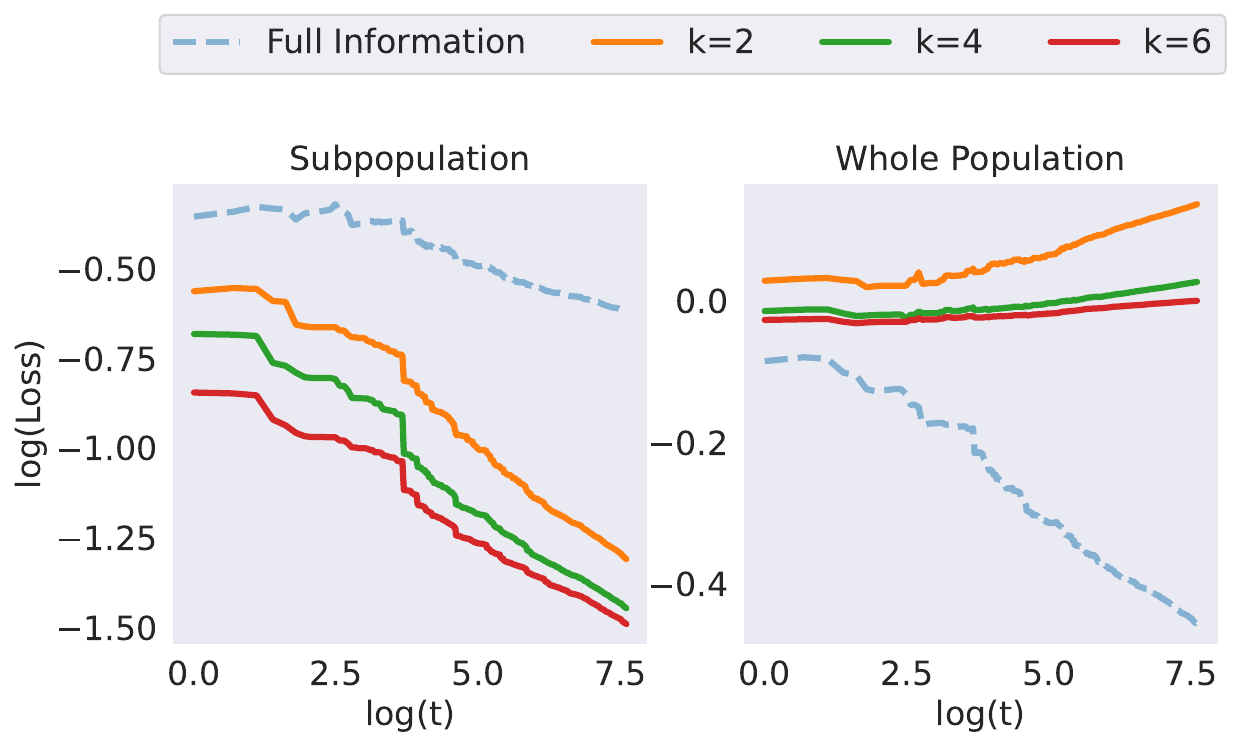}  
  \caption{Perfectly rational user: $\zeta = 0$}
\end{subfigure}
\begin{subfigure}{.5\textwidth}
  \centering
  \includegraphics[width=1\linewidth]{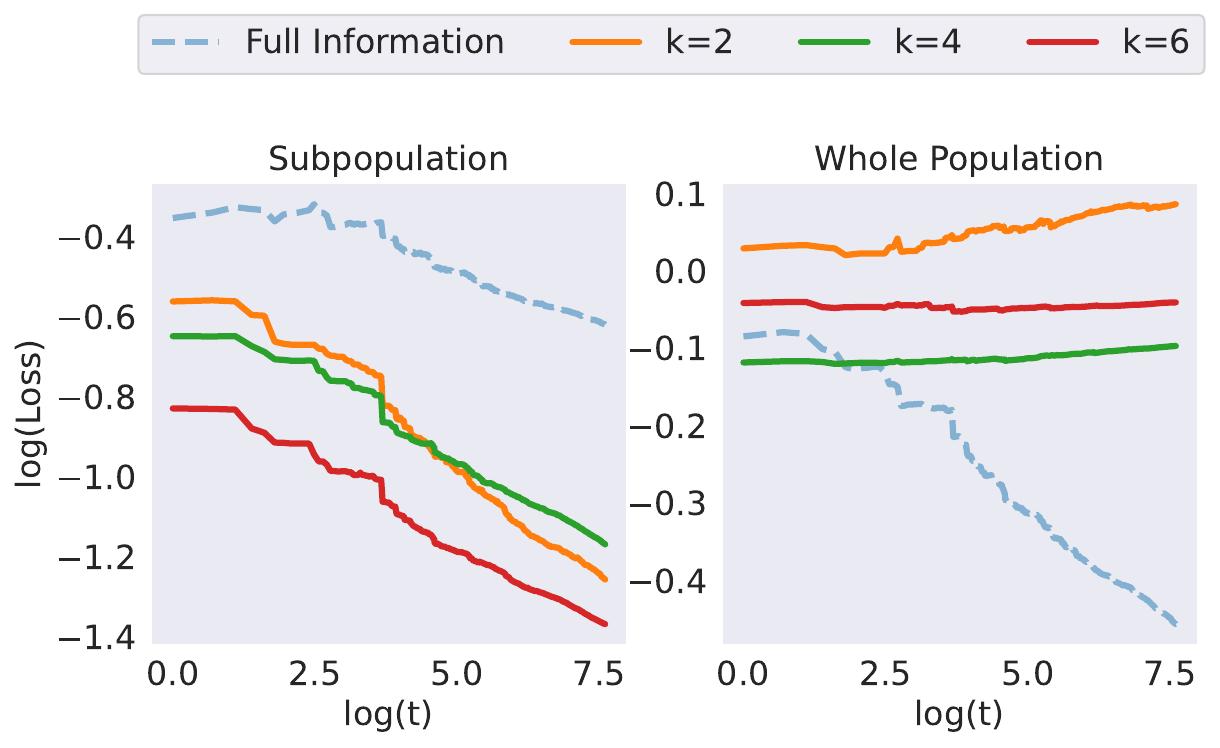}  
  \caption{Relatively rational user: $\zeta = 0.1$.}
\end{subfigure}
\caption{Log-log scaled loss of MSGD or Full Information on the model-specific subpopulation $\mathcal D_i(\Theta)$ and whole population $\mathcal P$ for the ACSEmployment task on census data. For MSGD, we illustrate results of different total number of services $k =2, 4,6$. }
\label{fig: vs full information on loss}
\end{figure}

\paragraph{MSGD in Different Settings}
\begin{figure}
\centerline{\includegraphics[width=1\textwidth]{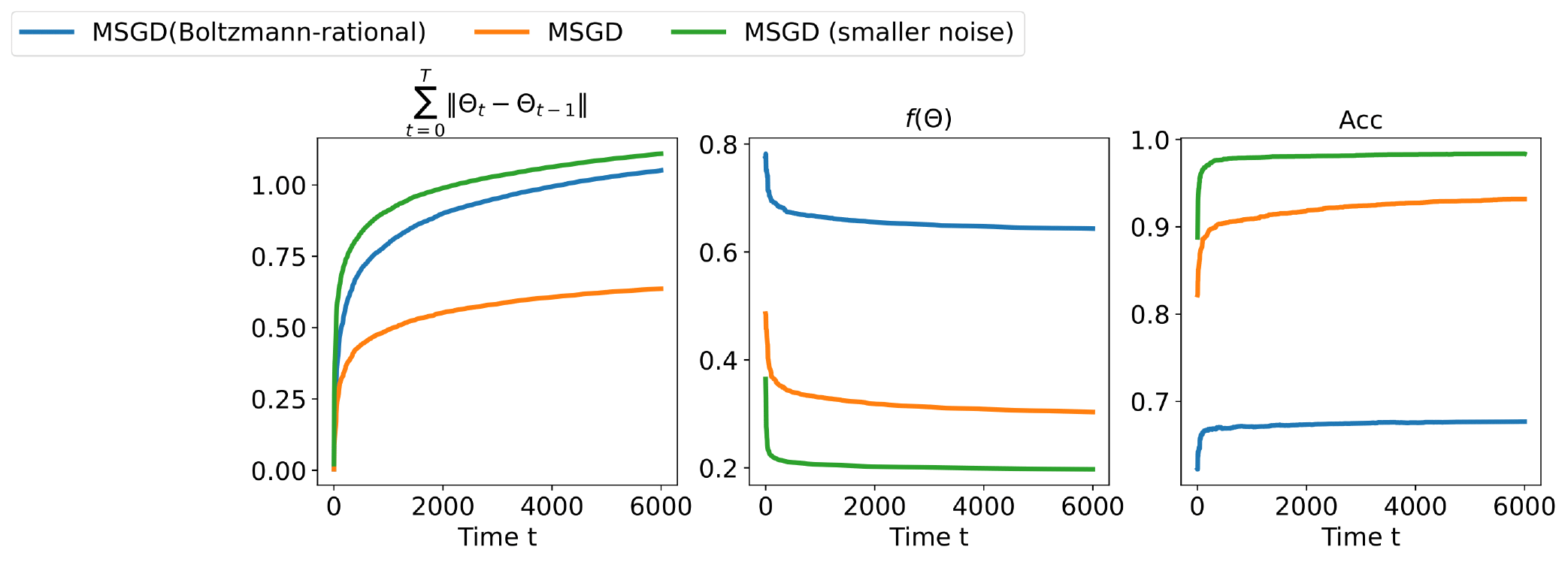}}
\caption{Comparing 
the iterates (left), objective function (middle) and prediction accuracy (right) of MSGD and two variants with Census dataset. }
\label{fig: more comparison}
\end{figure}
{We show in Figure \ref{fig: more comparison} the following two variants of the setting to demonstrate and better understand the advantage and limitations of MSGD.
\begin{enumerate}
    \item MSGD (Boltzmann-rational): When user behaves under  Boltzmann-rational model with $\alpha = 0.5$.
    \item MSGD (smaller noise): When we have more accurate user gradient. (Instead of one, 6 users arrive in each time step.)
\end{enumerate}
}
{The experiments are with Census dataset, for MSGD baseline, we use $\zeta= 0.1$.
From these additional experiments, we can see that
our MSGD algorithm can work well and converge even when user have more diverse behavior such as  Boltzmannn Rational model.
Having smaller noise on the online gradient can further improve the performance of our algorithm, which suggests that our algorithm, though designed for streaming user setting, can work well for both our setting and the more complete data setting that has been studied in previous papers.}
\paragraph{MSGD under Boltzmann-rational Model }
{
Though it is theoretically difficult to analyze the convergence of MSGD under Boltzmann-rational model,  we conduct additional empirical experiments of MSGD under Boltzmann-rational model on Census dataset with $\alpha =0, 0.5, 1$. The results are shown in figure \ref{fig: boltzmann rational model}, from which,  we do see that the iterates, $f(\Theta)$ and accuracy still converges, which indicates that our MSGD algorithm is robust enough to work well even when user behaviors deviates from our setting. }
\begin{figure}
\centerline{\includegraphics[width=1\textwidth]{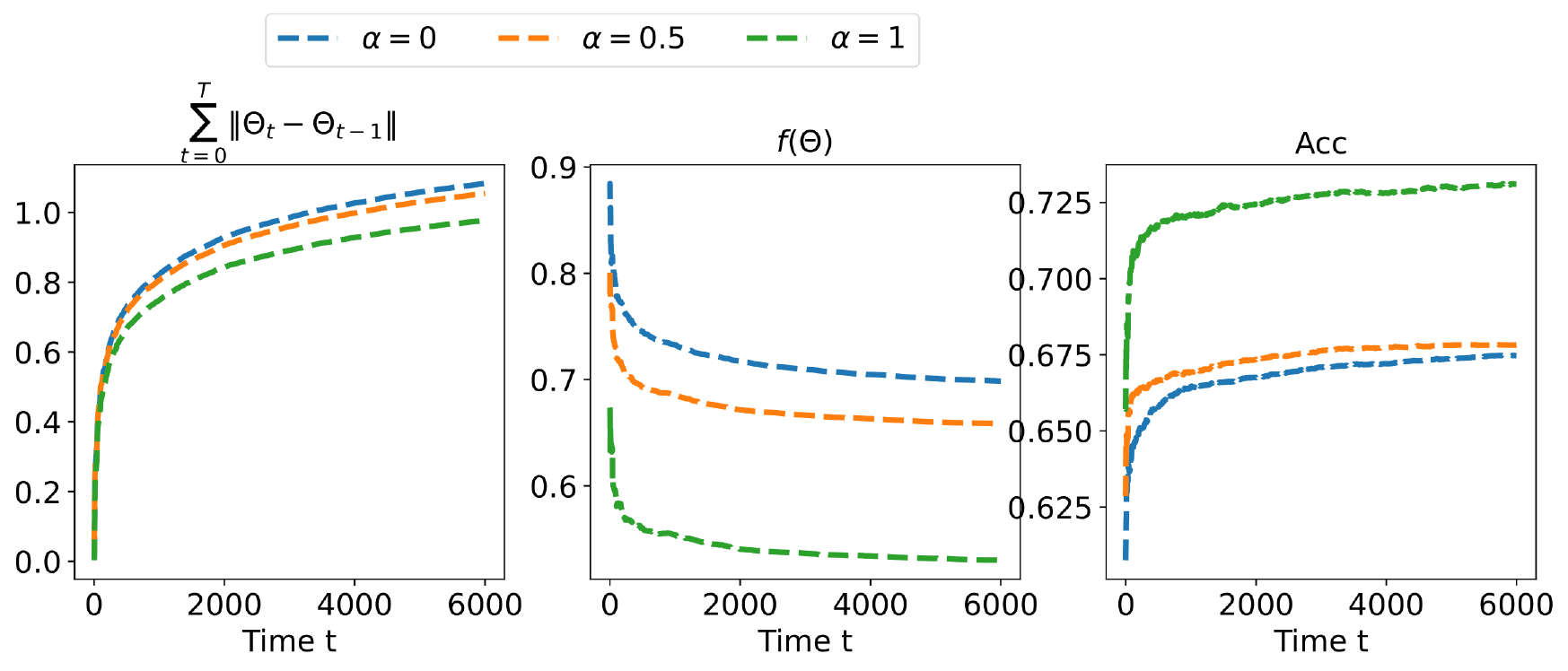}}
\caption{The convergence of the iterates (left), objective function (middle) and prediction accuracy (right) of MSGD under Boltzmann-rational model on Census dataset with $\alpha =0, 0.5, 1$.}
\label{fig: boltzmann rational model}
\end{figure}
\newpage
\section{Examples of Assumption \ref{ass: assumptions for proving a being Lipschitz}}\label{appedix: assumption intuition}

Generally, there are two kinds of scenarios where $\ell(x,\theta)$ and $\ell(x, \theta^{'})$ could be close to each other: One case is $\theta$ being close to $\theta^{'}$, the expression in Assumption \ref{ass: assumptions for proving a being Lipschitz} may only hold for small $d$. 
As long as $\theta$ are not arbitrarily close to $\theta^{'}$, the assumption states that we can still always find a small enough $d_0$ such that the expression holds. 
Another case is when $\theta$ is distinguishable from $\theta^{'}$, but there still exist some users who are ambiguous on which service to choose (i.e., ${\{x: |\ell(x, \theta)-\ell(x, \theta^{'})|<d}\}$). We give an example of these scenarios in Figure \ref{fig: assumption}.
Assumption \ref{ass: assumptions for proving a being Lipschitz} states that the volume of these ambiguous users can be controlled by $d$.

\begin{figure}[h]
\centerline{\includegraphics[width=0.7\textwidth]{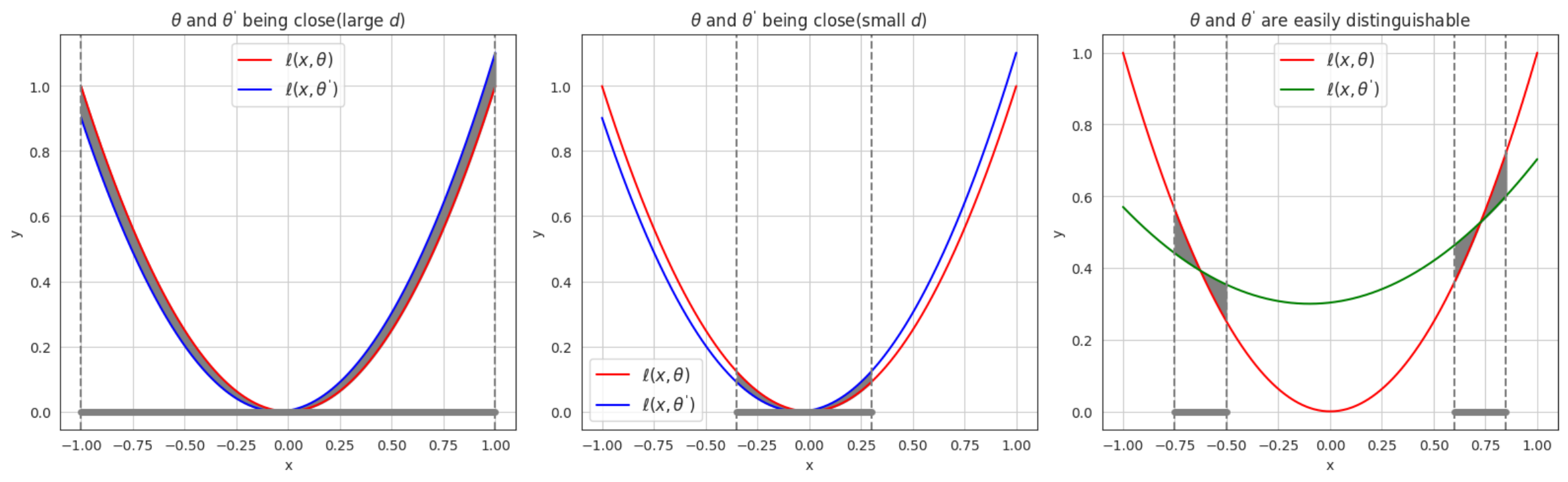}}
\caption{Examples of Assumption \ref{ass: assumptions for proving a being Lipschitz}. \textbf{Left}: example of $\theta$ being close to $\theta^{'}$, Assumption \ref{ass: assumptions for proving a being Lipschitz} is hard to hold with large $d$. \textbf{Middle}: $\theta$ being close to $\theta^{'}$ and Assumption \ref{ass: assumptions for proving a being Lipschitz} holds with sufficiently small $d$.
\textbf{Right}: $\theta$ is distinguishable from $\theta^{'}$, but there are still some users who are ambiguous about which service to choose (i.e., $\{{x: |\ell(x, \theta)-\ell(x, \theta^{'})|<d}\}$).
 }
\label{fig: assumption}
\end{figure}

\end{document}